%% file: root.tex
\documentclass{article}

 \PassOptionsToPackage{numbers, compress}{natbib}
\usepackage{jmlr2e} 
\usepackage{hyperref}       % hyperlinks
\usepackage{url}            % simple URL typesetting
\usepackage{booktabs}       % professional-quality tables
\usepackage{amsfonts}       % blackboard math symbols
\usepackage{nicefrac}       % compact symbols for 1/2, etc.
\usepackage{microtype}      % microtypography

\usepackage{algorithm}
\usepackage{algorithmic}
\usepackage{hyperref}
\usepackage{natbib}

\usepackage{amsmath,amssymb}
\usepackage{pgfplots}
\usepackage{color}
\usepackage{multirow}
\usepackage{rotating}
\usepackage{appendix}
\usepackage{tikz,pgfplots}
\usepackage{subcaption}

\input{mysymbol.sty}

\newtheorem{assumption}{\hspace{0pt}\bf Assumption}

\begin{document}

\jmlrheading{1}{2000}{1-48}{4/00}{10/00}{}

% Short headings should be running head and authors last names

\ShortHeadings{Escaping Saddle Points in Constrained Optimization}{Mokhtari, Ozdaglar, and Jadbabaie}
\firstpageno{1}

\title{Escaping Saddle Points in Constrained Optimization}

\author{\name Aryan Mokhtari  \email aryanm@mit.edu \\
    \name Asuman Ozdaglar \email asuman@mit.edu\\
       \name Ali Jadbabaie \email jadbabai@mit.edu  \\
       \addr Laboratory for Information and Decision Systems\\
       Massachusetts Institute of Technology\\
       Cambridge, MA 02139
       }
\editor{}

\maketitle

\thispagestyle{empty}

\vspace{15mm}

\input{abstract}
\vspace{10mm}

%%%%%%%%%%%%%%%%%%%%%%%%%%%%%%%%%%
%%%%%%%%%%%%%%%%%%%%%%%%%%%%%%%%%%
%%%%%   S  E  C  T  I  O  N    %%%%%%%%%%%%%%%%%%
%%%%%%%%%%%%%%%%%%%%%%%%%%%%%%%%%%
%%%%%%%%%%%%%%%%%%%%%%%%%%%%%%%%%%
\section{Introduction}

There has been a recent revival of interest in non-convex optimization, due to obvious applications in machine learning. While the modern history of the subject goes back six or seven decades, the recent attention to the topic stems from new applications as well as availability of modern analytical and  computational tools, providing  a new perspective on classical problems. Following this trend, in this paper we focus on the  problem of minimizing a smooth nonconvex function over a convex set as
 \begin{equation}\label{eq:main_problem}
 \text{minimize}\ f(\bbx), \qquad \text{subject to}\ \bbx\in \ccalC,
 \end{equation}
where $\bbx\in \reals^d$, $\ccalC \subset \reals^d$ is a closed convex set, and $f:\reals^d \to \reals$ is a twice continuously differentiable function over $\ccalC$. It is known that finding a global minimum of Problem~\eqref{eq:main_problem} is hard. Equally well-known is the fact that for certain nonconvex problems, all local minimizers are global. These include, for example, matrix completion \citep{DBLP:conf/nips/GeLM16}, phase retrieval \citep{DBLP:conf/isit/SunQW16}, and dictionary learning \citep{DBLP:journals/tit/SunQW17}. For such problems, finding a global minimum of \eqref{eq:main_problem} reduces to the problem of finding one of its local minima.

Given the well-known hardness results in finding stationary points, recent focus has shifted in characterizing approximate stationary points. When the objective function $f$ is convex, finding an $\eps$-first-order stationary point is often sufficient since it leads to finding an approximate local (and hence global) minimum. However, in the nonconvex setting, even when the problem is unconstrained, i.e., $\ccalC= \reals^d$, convergence to a first-order stationary point (FOSP) is not enough as the critical point to which convergence is established might be a saddle point. It is therefore natural  to look at higher order derivatives and search for a second-order stationary points. Indeed, under the assumption that all the saddle points are strict (formally defined later), in both unconstrained and constrained settings, convergence to a second-order stationary point (SOSP) implies convergence to a local minimum. While convergence to an SOSP has been thoroughly investigated in the recent literature for the unconstrained setting, the overall complexity for the constrained setting has not been studied yet.

%Solving the optimization problem in \eqref{eq:main_problem} is NP-hard in general, even for the unconstrained case
%Despite this complexity challenge, in many nonconvex problems all local minimizers are global, or the performance by a local minimum is comparable to the global minimum. Therefore, instead of solving the program in \eqref{eq:main_problem}, one might try to find a local minimum. However, even finding a local minimum might be hard as there might be many saddle points. 

%\red{a couple of more sentences to justify convergence to a second-order stationary point.\\
%Then, explain why this problem is challenging when we have a convex constraint. Why the traditional algorithms do not work anymore? Answer: the quadratic problem can not be solved since it is not the eigenvalue problem, perturbation may lead to an infeasible point and probability of being feasible and useful is very small. }

\textbf{Contributions.}
Our main contribution is to propose a generic framework which generates a sequence of iterates converging to an approximate second-order stationary point for the constrained nonconvex problem in \eqref{eq:main_problem}, when the convex set $\ccalC$ has a specific structure that allows for approximate minimization of a quadratic loss over the feasible set.  The proposed framework consists of two main stages: First, it utilizes first-order information to reach a first-order stationary point; next, it incorporates second-order information to escape from a stationary point if it is a local maximizer or a strict saddle point. 
We show that  the proposed approach leads to an $(\eps,\gamma)$-second-order stationary point (SOSP) for Problem \eqref{eq:main_problem} (check Definition~\ref{def_sosp_constrained}). The proposed approach utilizes advances in constant-factor optimization of nonconvex quadratic programs \citep{ye1992affine,fu1998approximation,tseng2003further} that find a $\rho$-approximate solution over $\ccalC$ in polynomial time, where $\rho$ is a positive constant smaller than $1$ that depends on the structure of $\ccalC$. When such approximate solution exists, the sequence of iterates generated by the proposed framework reaches an $(\eps,\gamma)$-SOSP of Problem~\eqref{eq:main_problem}  in at most $\mathcal{O}(\max\{\eps^{-2},\rho^{-3}\gamma^{-3}\})$ iterations. 

We show that quadratic constraints satisfy the required condition for the convex set $\ccalC$ if the objective function Hessian $\nabla^2 f$ has a specific structure over the convex set $\ccalC$ (formally described later). For this case, we show that it is possible to achieve an $(\eps,\gamma)$-SOSP after at most $\mathcal{O}(\max\{\tau\eps^{-2} ,d^3m^7\gamma^{-3}\})$ arithmetic operations, where $d$ is the dimension of the problem and $\tau$ is the number of required arithmetic operations to solve a linear program over $\ccalC$ or to project a point onto $\ccalC$. We further extend our results to the stochastic setting and show that we can reach an $(\eps,\gamma)$-SOSP after computing at most $\mathcal{O}(\max\{\eps^{-4},\eps^{-2}\rho^{-4}\gamma^{-4},  \rho^{-7}\gamma^{-7}\})$ stochastic gradients and $\mathcal{O}(\max\{\eps^{-2}\rho^{-3}\gamma^{-3},\rho^{-5}\gamma^{-5}\})$ stochastic Hessians.

%\red{TTT is the complexity of solving a linear program over $\ccalC$ or computing the Euclidean projection of a point onto $\ccalC$.}

%\textbf{Notation}.
%Vectors are written as lowercase $\bbx$ and matrices as uppercase $\bbA$. We use $\|\bbx\|$ and  $\|\bbA\|$ to denote the Euclidean norm of vector $\bbx$ and matrix $\bbA$, respectively. %Given a positive definite matrix ${\bbM_i}$, the weighted matrix norm $\|\bbA\|_{\bbM_i}$ is defined as $\|\bbA\|_{\bbM_i}:= \|{\bbM_i}\bbA{\bbM_i}\|_\bbF$, where $\|.\|_\bbF$ is the Frobenius norm. 
%Given a function $f$ its gradient and Hessian at point $\bbx$ are denoted as $\nabla f(\bbx)$ and $\nabla^2 f(\bbx)$, respectively. The Euclidean projection $\pi_\ccalC$ of a point $\bby$ onto a convex set $\ccalC$ is denoted by $\pi_\ccalC(\bby)$ and defined as $
%\pi_\ccalC(\bby):= \argmin_{\bbx \in \ccalC}\{\|\bbx-\bby\|\}.$

%%%%%%%%%%%%%%%%%%%%%%%%%%%%%%%%%%
%%%%%%%%%%%%%%%%%%%%%%%%%%%%%%%%%%
%%%%%  S  U  B  -- S  E  C  T  I  O  N    %%%%%%%%%%%%%
%%%%%%%%%%%%%%%%%%%%%%%%%%%%%%%%%%
%%%%%%%%%%%%%%%%%%%%%%%%%%%%%%%%%%
\subsection{Related work} 

%The literature is full of results on constyrained and unconstrained minimization of smooth, nonconvex functions. Below we will review some of the most relevant ones tot his work.\\

\textbf{Unconstrained case}. The rich literature on nonconvex optimization provides a plethora of algorithms for reaching stationary points of a smooth \textit{unconstrained} minimization problem. Convergence to first-order stationary points (FOSP) has been widely studied for both deterministic \citep{nesterov2013introductory, DBLP:conf/stoc/AgarwalZBHM17, DBLP:journals/corr/CarmonDHS16, DBLP:conf/icml/CarmonDHS17, carmon2017lowerpart1, carmon2017lowerpart2} and stochastic settings \citep{DBLP:conf/cdc/ReddiSPS16,DBLP:conf/icml/ReddiHSPS16,DBLP:conf/icml/ZhuH16,DBLP:conf/nips/LeiJCJ17}. Stronger results which indicate convergence to an SOSP are also established. Numerical optimization methods such as trust-region methods \citep{DBLP:journals/jc/CartisGT12,DBLP:journals/mp/CurtisRS17,DBLP:journals/jgo/MartinezR17} and cubic regularization algortihms  \citep{DBLP:journals/mp/NesterovP06,DBLP:journals/mp/CartisGT11,DBLP:journals/mp/CartisGT11a} can reach an approximate second-order stationary point in a finite number of iterations; however, typically the computational complexity of each iteration could be relatively large due to the cost of solving trust-region or regularized cubic subproblems. Recently, a new line of research has emerged that focuses on the overall computational cost to achieve an SOSP. These results build on the idea of escaping from strict saddle points with perturbing the iterates by injecting a properly chosen noise  \citep{DBLP:conf/colt/GeHJY15, DBLP:conf/icml/Jin0NKJ17, DBLP:journals/corr/abs-1711-10456}, or by updating the iterates using the eigenvector corresponding to the smallest eigenvalue of the Hessian \citep{DBLP:journals/corr/CarmonDHS16, DBLP:journals/corr/abs-1708-08694, xu2017first,royer2017complexity,DBLP:conf/stoc/AgarwalZBHM17,DBLP:conf/aistats/ReddiZSPBSS18,paternain2017second}. 

\textbf{Constrained case}. Asymptotic convergence to first-order and second-order stationary points for the constrained optimization problem in~\eqref{eq:main_problem} has been studied in the numerical optimization community \citep{burke1990convergence,conn1993global,facchinei1998convergence,di2005convergence}. Recently, finite-time analysis for convergence to an FOSP of the generic smooth constrained  problem in~\eqref{eq:main_problem} has received a lot of attention. In particular, \citep{lacoste2016convergence} shows that the sequence of iterates generated by the update of Frank-Wolfe converges to an $\eps$-FOSP after $\mathcal{O}(\eps^{-2})$ iterations. %However, to prove this claim the stepsize either should be chosen according to a line-search scheme or must be smaller than a threshold which depends on the curvature of the function $f$. In this paper, we show that with a constant step
The authors of \citep{ghadimi2016mini} consider norm of gradient mapping as a measure of non-stationarity and show that the projected gradient method has the same complexity of $\mathcal{O}(\eps^{-2})$. Similar result for the accelerated projected gradient method is also shown \citep{ghadimi2016accelerated}. Adaptive cubic regularization methods in \citep{cartis2012adaptive,cartis2013evaluation,cartis2015evaluation} improve these results using second-order information and obtain an $\eps$-FOSP of Problem~\eqref{eq:main_problem} after at most $\mathcal{O}(\eps^{-3/2})$ iterations. Finite time analysis for convergence to an SOSP has also been studied for linear constraints. In particular, \citep{bian2015complexity} studies convergence to an SOSP of \eqref{eq:main_problem} when the set $\ccalC$ is a linear constraint of the form $\bbx\geq 0$ and propose a trust region interior point method that obtains an $(\eps,\sqrt{\eps})$-SOSP in $\mathcal{O}(\eps^{-3/2})$ iterations. The work in \citep{haeser2017optimality} extends their results to the case that the objective function is potentially not differentiable or not twice differentiable on the boundary of the feasible region. The authors in
 \citep{cartis2017second} focus on the general convex constraint case and introduce a trust region algorithm that requires $\mathcal{O}(\epsilon^{-3})$ iterations to obtain an SOSP; however, each iteration of their proposed method requires access to the exact solution of a nonconvex quadratic program (finding its global minimum) which, in general, could be computationally prohibitive. To the best of our knowledge, our paper provides the first finite-time overall computational complexity analysis for reaching an SOSP of Problem~\eqref{eq:main_problem}.

%%%%%%%%%%%%%%%%%%%%%%%%%%%%%%%%%%
%%%%%%%%%%%%%%%%%%%%%%%%%%%%%%%%%%
%%%%%   S  E  C  T  I  O  N    %%%%%%%%%%%%%%%%%%
%%%%%%%%%%%%%%%%%%%%%%%%%%%%%%%%%%
%%%%%%%%%%%%%%%%%%%%%%%%%%%%%%%%%%

%%%%%%%%%%%%%%%%%%%%%%%%%%%%%%%%%%
%%%%%%%%%%%%%%%%%%%%%%%%%%%%%%%%%%
%%%%%   S  E  C  T  I  O  N    %%%%%%%%%%%%%%%%%%
%%%%%%%%%%%%%%%%%%%%%%%%%%%%%%%%%%
%%%%%%%%%%%%%%%%%%%%%%%%%%%%%%%%%%
\section{Preliminaries and Definitions}

In the case of \textit{unconstrained} minimization of the objective function $f$, the first-order and second-order necessary conditions for a point $\bbx^*$ to be a local minimum of that are defined as
%%%
%\vspace{1mm}
%\begin{equation}\label{eq:unconstrained_necess_conditions}
$\nabla f(\bbx^*)=\bb0_d$ and $\nabla^2 f(\bbx^*)\succeq \bb0_{d\times d}$,
%\end{equation}
%%%
respectively. If a point satisfies these conditions it is called a \textit{second-order stationary point} (SOSP). If the second condition becomes strict, i.e., $\nabla^2 f(\bbx)\succ \bb0$, then we recover the sufficient conditions for a local minimum. However, to derive finite time convergence bounds for achieving an SOSP, these conditions should be relaxed. In other words, the goal should be to find an \textit{approximate} SOSP where the approximation error can be arbitrarily small. For the case of unconstrained minimization, a point $\bbx^*$ is called an $(\eps,\gamma)$-second-order stationary point if it satisfies 
%%%
$
\|\nabla f(\bbx^*)\| \leq \eps$ and $\nabla^2 f(\bbx^*)\succeq -\gamma \bbI_d,$
%%%
where $\eps$ and $\gamma$ are arbitrary positive constants. 
%In the following definition we formally define strict saddle points for the unconstrained version of Problem~\eqref{eq:main_problem}, i.e., when $\ccalC=\reals^d$.
%
%%%%%%%%%%%%%%%%%%%%%%%%%%%%%%%%%%%
%%%%%%%%%%%%%%%%%%%%%%%%%%%%%%%%%%%
%%%%%%  D  E   F  I  N  I  T  I  O  N    %%%%%%%%%%%%%%%
%%%%%%%%%%%%%%%%%%%%%%%%%%%%%%%%%%%
%%%%%%%%%%%%%%%%%%%%%%%%%%%%%%%%%%%
%\begin{definition}\label{strict_saddle_unconstrained}
%Consider Problem~\eqref{eq:main_problem}  when $\ccalC=\reals^d$. Then, $\bbx$ is called a $\delta$-strict saddle point if it is a saddle point, i.e., $\nabla f(\bbx)=0$ and $\nabla^2 f(\bbx)$ is indefinite, and the smallest eigenvalue of the Hessian $\nabla^2 f(\bbx)$ evaluated at $\bbx$ is strictly smaller than $-\delta$, i.e., $\lambda_{min}(\nabla^2 f(\bbx))< -\delta$.
%\end{definition}
%%%%
%
%Using the definitions of a $\delta$-strict saddle and an $(\eps,\gamma)$-SOSP we obtain that an $(\eps,\gamma)$-SOSP is a local minimum for the unconstrained optimization problem if all the saddle points are $\delta$-strict and the condition $\gamma\leq \delta$ is satisfied. 
%
%
To study the constrained setting, we first state the necessary conditions for a local minimum of problem \eqref{eq:main_problem}.

%%%%%%%%%%%%%%%%%%%%%%%%%%%%%
%%%%%%%%%%%%%%%%%%%%%%%%%%%%%
%%%%   PROPOSITION    %%%%%%%%%%%%%%%
%%%%%%%%%%%%%%%%%%%%%%%%%%%%%
%%%%%%%%%%%%%%%%%%%%%%%%%%%%%
\begin{proposition}[\citep{bertsekas1999nonlinear}]\label{prop:nec_conds}
If $\bbx^*\in \ccalC$ is a local minimum of the function $f$ over the convex set $\ccalC$, then 
%%% 
\begin{align}
&\nabla f(\bbx^*)^\top(\bbx-\bbx^*)\geq 0 ,\quad \forall \ \bbx \in \ccalC, \label{eq:nec_cond_first_order}\\
& (\bbx-\bbx^*)^\top\nabla^2 f(\bbx^*)(\bbx-\bbx^*) \geq 0 ,\quad \forall \ \bbx \in \ccalC\  \ \st\ \nabla f(\bbx^*)^\top(\bbx-\bbx^*)=0.\label{eq:nec_cond_second_order}
\end{align}
%
%
%\begin{equation}\label{eq:nec_cond_first_order}
%\nabla f(\bbx^*)^\top(\bbx-\bbx^*)\geq 0, 
%\end{equation}
%%%%
%and for all $\bbx \in \ccalC$ satisfying $\nabla f(\bbx^*)^\top(\bbx-\bbx^*)= 0$ we have
%%%% 
%\begin{equation}\label{eq:nec_cond_second_order}
%(\bbx-\bbx^*)^\top\nabla^2 f(\bbx^*)(\bbx-\bbx^*) \geq  0.
%\end{equation}
%%%%
\end{proposition}

%
%%%%%%%%%%%%%%%%%%%%%%%%%%%%%%
%%%%%%%%%%%%%%%%%%%%%%%%%%%%%%
%%%%%   PROPOSITION    %%%%%%%%%%%%%%%
%%%%%%%%%%%%%%%%%%%%%%%%%%%%%%
%%%%%%%%%%%%%%%%%%%%%%%%%%%%%%
%\begin{proposition}
%(Sufficient Conditions) Assume that the convex set $\ccalC$ is a polyhedral. If for all $\bbx \in \ccalC$ the inequality 
%%%% 
%\begin{equation}
%\nabla f(\bbx^*)^\top(\bbx-\bbx^*)\geq 0
%\end{equation}
%%%%
%holds, and for all $\bbx \in \ccalC$ satisfying $\nabla f(\bbx^*)^\top(\bbx-\bbx^*)= 0$ the condition
%%%% 
%\begin{equation}
%(\bbx-\bbx^*)^\top\nabla^2 f(\bbx^*)(\bbx-\bbx^*) > 0
%\end{equation}
%%%%
%is satisfied, then the vector $\bbx^*\in \ccalC$ is a local minimum of the function $f$ over the convex set $\ccalC$.
%\end{proposition}
%
%%%%%%%%%%%%%%%%%%%%%%%%%%%%%%
%%%%%%%%%%%%%%%%%%%%%%%%%%%%%%
%%%%%    PROOF     %%%%%%%%%%%%%%%%%%
%%%%%%%%%%%%%%%%%%%%%%%%%%%%%%
%%%%%%%%%%%%%%%%%%%%%%%%%%%%%%
%\vspace{2mm}
%\begin{proof}:
%%We prove the claim by contradiction.  Assume that $\bbx^*$ is not a local minimum of the function $f$ over the convex set $\ccalC$. Therefore, there exists at least a sequence of feasible vectors $\bbx_t\in \ccalC$ which converges to $\bbx^*$ and $f(\bbx^k)<f(\bbx^*)$ for all $k$.
%\end{proof}

%The expressions in \eqref{eq:nec_cond_first_order} and \eqref{eq:nec_cond_second_order} indicate the necessary conditions for a local minimum $\bbx^*$ of Problem~\eqref{eq:main_problem}. 
The conditions in \eqref{eq:nec_cond_first_order} and \eqref{eq:nec_cond_second_order} are the first-order and second-order necessary optimality conditions, respectively. By making the inequality in \eqref{eq:nec_cond_second_order} strict, i.e., $(\bbx-\bbx^*)^\top\nabla^2 f(\bbx^*)(\bbx-\bbx^*)>0$, we recover the sufficient conditions for a local minimum when $\ccalC$ is a polyhedral \citep{bertsekas1999nonlinear}. Further, if the inequality in~\eqref{eq:nec_cond_second_order} is replaced by $(\bbx-\bbx^*)^\top\nabla^2 f(\bbx^*)(\bbx-\bbx^*) \geq  \delta \|\bbx-\bbx^*\|^2$ for some $\delta>0$, we obtain the sufficient conditions for a local minimum of Problem \eqref{eq:main_problem} for any convex constraint $\ccalC$; see \citep{bertsekas1999nonlinear}. If a point $\bbx^*$ satisfies the conditions in~\eqref{eq:nec_cond_first_order} and \eqref{eq:nec_cond_second_order} it is an SOSP of Problem~\eqref{eq:main_problem}. %Similar to the unconstrained setting, if all the saddles are strict, an SOSP of Problem~\eqref{eq:main_problem} is a local minimum.
As in the unconstrained setting, the first-order and second-order optimality conditions may not be satisfied in finite number of iterations, and we focus on finding an approximate SOSP. 

%%%%%%%%%%%%%%%%%%%%%%%%%%%%%%%%%%
%%%%%%%%%%%%%%%%%%%%%%%%%%%%%%%%%%
%%%%%  D  E   F  I  N  I  T  I  O  N    %%%%%%%%%%%%%%%
%%%%%%%%%%%%%%%%%%%%%%%%%%%%%%%%%%
%%%%%%%%%%%%%%%%%%%%%%%%%%%%%%%%%%
\begin{definition}\label{def_sosp_constrained}
Recall the twice continuously differentiable function $f:\reals^d\to \reals$ and the convex closed set $\ccalC\subset \reals^d$ introduced in Problem~\eqref{eq:main_problem}. We call $\bbx^*\in \ccalC$ an $(\eps,\gamma)$-second order stationary point of Problem~\eqref{eq:main_problem} if the following conditions are satisfied.
%%%
\begin{align}
&\nabla f(\bbx^*)^\top(\bbx-\bbx^*)\geq -\eps ,\quad \forall \ \bbx \in \ccalC,\\
& (\bbx-\bbx^*)^\top\nabla^2 f(\bbx^*)(\bbx-\bbx^*) \geq -\gamma ,\quad \forall \ \bbx \in \ccalC\  \ \st\ \nabla f(\bbx^*)^\top(\bbx-\bbx^*)=0.
\end{align}
If a point only satisfies the first condition, we call it an $\eps$-first order stationary point.
\end{definition}
%%%

We further formally define strict saddle points for the constrained optimization problem in~\eqref{eq:main_problem}.

%%%%%%%%%%%%%%%%%%%%%%%%%%%%%%%%%%
%%%%%%%%%%%%%%%%%%%%%%%%%%%%%%%%%%
%%%%%  D  E   F  I  N  I  T  I  O  N    %%%%%%%%%%%%%%%
%%%%%%%%%%%%%%%%%%%%%%%%%%%%%%%%%%
%%%%%%%%%%%%%%%%%%%%%%%%%%%%%%%%%%
\begin{definition}\label{strict_saddle_constrained}
A point $\bbx^*\in \ccalC$ is a $\delta$-strict saddle point of Problem~\eqref{eq:main_problem} if (i) for all $\bbx \in \ccalC$ the condition $\nabla f(\bbx^*)^\top(\bbx-\bbx^*)\geq 0$ holds, and (ii) there exists a point $\bby$ such that 
\begin{equation}
(\bby-\bbx^*)^\top\nabla^2 f(\bbx^*)(\bby-\bbx^*) < -\delta ,\quad \ \bby \in \ccalC\  \text{and}\ \nabla f(\bbx^*)^\top(\bby-\bbx^*)=0.
\end{equation}
\end{definition}
%%%

According to Definitions \ref{def_sosp_constrained} and \ref{strict_saddle_constrained} if all saddle points are $\delta$-strict and $\gamma\leq \delta$, any $(\eps,\gamma)$-SOSP of Problem~\eqref{eq:main_problem} is an approximate local minimum. %In this case, by sending $\eps\to0$ the limit point is a local minimum.

We emphasize that in this paper we do not assume that all saddles are strict to prove convergence to an SOSP. We formally defined strict saddles just to clarify that if all the saddles are strict then convergence to an approximate SOSP is equivalent to convergence to an approximation local minimum.

 Our goal throughout the rest of the paper is to design an algorithm which finds an $(\eps,\gamma)$-SOSP of Problem~\eqref{eq:main_problem}. To do so, we first assume the following conditions are satisfied.

%%%%%%%%%%%%%%%%%%%%%%%%%%%%%%%%%%
%%%%%%%%%%%%%%%%%%%%%%%%%%%%%%%%%%
%%%%%  D  E   F  I  N  I  T  I  O  N    %%%%%%%%%%%%%%%
%%%%%%%%%%%%%%%%%%%%%%%%%%%%%%%%%%
%%%%%%%%%%%%%%%%%%%%%%%%%%%%%%%%%%
\begin{assumption}\label{assumption:lip_grad}
The gradients $\nabla f$ are $L$-Lipschitz continuous over the set $\ccalC$, i.e., for  any $\bbx,\tbx\in \ccalC$, 
%%%%
\begin{align}
\| \nabla f(\bbx)-\nabla f(\tbx)\| \leq L\|\bbx-\tbx\|.
\end{align}
%%%%%
\end{assumption}
%%%%%%%%%%%%%%%%%%%%%%%%%%

%%%%%%%%%%%%%%%%%%%%%%%%%%%%%%%%%%
%%%%%%%%%%%%%%%%%%%%%%%%%%%%%%%%%%
%%%%%  D  E   F  I  N  I  T  I  O  N    %%%%%%%%%%%%%%%
%%%%%%%%%%%%%%%%%%%%%%%%%%%%%%%%%%
%%%%%%%%%%%%%%%%%%%%%%%%%%%%%%%%%%
\begin{assumption}\label{assumption:lip_hessian}
 The Hessians $\nabla^2 f$ are $M$-Lipschitz continuous over the set $\ccalC$, i.e., for any $\bbx,\tbx\in \ccalC$
%%%%
\begin{align}
\| \nabla^2 f(\bbx)-\nabla^2 f(\tbx)\| \leq M\|\bbx-\tbx\|.
\end{align}
%%%%%
\end{assumption}
%%%%%%%%%%%%%%%%%%%%%%%%%%

%%%%%%%%%%%%%%%%%%%%%%%%%%%%%%%%%%
%%%%%%%%%%%%%%%%%%%%%%%%%%%%%%%%%%
%%%%%  D  E   F  I  N  I  T  I  O  N    %%%%%%%%%%%%%%%
%%%%%%%%%%%%%%%%%%%%%%%%%%%%%%%%%%
%%%%%%%%%%%%%%%%%%%%%%%%%%%%%%%%%%
\begin{assumption}\label{assumption:bounded_diameter}
 The diameter of the compact convex set $\ccalC$ is upper bounded by a constant $D$, i.e., 
%%%%
\begin{align}
 \max_{\bbx,\tbx \in \ccalC}\{\|\bbx-\hbx\|\}\leq D.
\end{align}
%%%%%
\end{assumption}
%%%%%%%%%%%%%%%%%%%%%%%%%%
%%%%%%%%%%%%%%%%%%%%%%%%%%%%%%%%%%%
%%%%%%%%%%%%%%%%%%%%%%%%%%%%%%%%%%%
%%%%%%  D  E   F  I  N  I  T  I  O  N    %%%%%%%%%%%%%%%
%%%%%%%%%%%%%%%%%%%%%%%%%%%%%%%%%%%
%%%%%%%%%%%%%%%%%%%%%%%%%%%%%%%%%%%
%\begin{definition} The Euclidean projection $\pi_\ccalC$ of the point $\bby$ to the convex set $\ccalC$ is defined as the unique solution of 
%%%%%
%\begin{align}
%\pi_\ccalC(\bby):= \argmin_{\bbx \in \ccalC}\{\|\bbx-\bby\|\}.
%\end{align}
%%%%%%
%\end{definition}
%%%%%%%%%%%%%%%%%%%%%%%%%%%

%\red{Explain the assumptions and mention that these conditions are customary.}

\section{Main Result}

In this section, we introduce a generic framework to reach an $(\eps,\gamma)$-SOSP of the non-convex function $f$ over the convex set $\ccalC$, when $\ccalC$ has a specific structure as we describe below. In particular, we focus on the case when we can solve a quadratic program (QP) of the form 
%%%
 %%
 \begin{align}\label{eq:quad_problem_template}
 & \text{minimize}\quad \bbx^\top\bbA\bbx+\bbb^\top\bbx +c \qquad  \text{subject to}\quad \bbx\in \ccalC,
 \end{align}
%%%
 up to a constant factor $\rho\leq 1$ in a finite number of arithmetic operations. Here, $\bbA\in\reals^d$ is a symmetric matrix, $\bbb\in \reals^d$ is a vector, and $c\in \reals$ is a scalar. %Indeed, if $\bbA$ is positive semidefinite, we face a convex QP and the problem can be exactly solved.
  To clarify the notion of solving a problem up to a constant factor $\rho$, consider $\bbx^*$ as a global minimizer of \eqref{eq:quad_problem_template}. Then, we say Problem \eqref{eq:quad_problem_template} is solved up to a constant factor $ \rho\in (0,1]$ if we have found a feasible solution $\tbx\in \ccalC$ such that  
%%%
\begin{equation}\label{rho_opt_cond}
{\bbx^*}^\top\bbA\bbx^*+\bbb^\top\bbx^*+c\ \leq\ \tbx^\top\bbA\tbx+\bbb^\top\tbx +c \ \leq\
 \rho({\bbx^*}^\top\bbA\bbx^*+\bbb^\top\bbx^*+c).
\end{equation}
%%%
Note that here w.l.o.g. we have assumed that the optimal objective function value ${\bbx^*}^\top\bbA\bbx^*+\bbb^\top\bbx^*+c$ is non-positive. Larger constant $\rho$ implies that the approximate solution is more accurate. If $\tbx$ satisfies  the condition in \eqref{rho_opt_cond}, we call it a $\rho$-approximate solution of Problem~\eqref{eq:quad_problem_template}. Indeed, if $\rho=1$ then $\tbx$ is a global minimizer of Problem~\eqref{eq:quad_problem_template}.

In Algorithm \ref{alg_generic_framework}, we introduce a generic framework that achieves an $(\eps,\gamma)$-SOSP of Problem~\eqref{eq:main_problem} whose running time is polynomial in $\eps^{-1}$, $\gamma^{-1}$, $\rho^{-1}$ and $d$, when we can find a $\rho$-approximate solution of a quadratic problem of the form \eqref{eq:quad_problem_template} in a time that is polynomial in $d$.
The proposed scheme consists of two major stages. In the first phase, as mentioned in Steps 2-4, we use a first-order update, i.e., a gradient-based update, to find an $\eps$-FOSP, i.e., we update the decision variable $\bbx$ according to a first-order update until we reach a point $\bbx_t$ that satisfies the condition 
%%%
\begin{equation}\label{eq:first_order_stat_cond}
\nabla f(\bbx_t)^\top(\bbx-\bbx_t)\geq -\eps, \quad \forall \ \bbx\in \ccalC .
\end{equation}
%%%
In Section \ref{sec:first_order_section}, we study in detail projected gradient descent and conditional gradient algorithms for the first order phase of the proposed framework. Interestingly, both of these algorithms require at most $\mathcal{O}(\eps^{-2})$ iterations to reach an $\eps$-first order stationary point. 

The second stage of the proposed scheme uses second-order information of the objective function~$f$ to escape from the stationary point if it is a local maximum or a strict saddle point. To be more precise, if we assume that $\bbx_t$ is a feasible point satisfying the condition \eqref{eq:first_order_stat_cond}, we then aim to find a descent direction by solving the following quadratic program 
%%%
 %%
 \begin{align}\label{eq:quad_problem}
 & \text{minimize}\quad q(\bbu) := (\bbu-\bbx_t)^\top\nabla^2 f(\bbx_t)(\bbu-\bbx_t)  \nonumber\\
 & \text{subject to}\quad \bbu\in \ccalC, \quad \nabla f(\bbx_t)^\top(\bbu-\bbx_t)=0,
 \end{align}
%%%
 up to a constant factor $\rho$ where $\rho\in(0,1]$. To be more specific, if we define $q(\bbu^*)$ as the optimal objective function value of the program in \eqref{eq:quad_problem}, we focus on the cases that we can obtain a feasible point $\bbu_t$ which is a  $\rho$-approximate solution of Problem~\eqref{eq:quad_problem}, i.e., $\bbu_t\in \ccalC$ and
\begin{equation}\label{eq:rho_approx_solution}
q(\bbu^*)\ \leq\ q(\bbu_t)\ \leq \ \rho \ q(\bbu^*).
\end{equation}
%%%
%Indeed, since $\bbx_t$ is a feasible point for the program in \eqref{eq:quad_problem} and $q(\bbx_t)=0$ it follows that $OPT\leq 0$. 
The problem formulation in \eqref{eq:quad_problem} can be transformed into the quadratic program in \eqref{eq:quad_problem_template}; see Section~\ref{sec:second_stage} for more details. Note that the constant $\rho$ is independent of $\eps$, $\gamma$, and $d$ and only depends on the structure of the convex set $\ccalC$. For instance, if $\ccalC$ is defined in terms of $m$ quadratic constraints one can find a $\rho=m^{-2}$ approximate solution of \eqref{eq:quad_problem} after at most $\mathcal{\tilde{O}}(md^3)$ arithmetic operations (Section~\ref{sec:second_stage}).

 After computing a feasible point $\bbu_t$ satisfying the condition in \eqref{eq:rho_approx_solution}, we check the quadratic objective function value at the point $\bbu_t$, and if the inequality $q(\bbu_t)<-\rho\gamma$ holds, we follow the update
%%%
\begin{equation}\label{eq:second_stage_update}
\bbx_{t+1} = (1-\sigma) \bbx_{t} + \sigma \bbu_t,
\end{equation}
where $\sigma$ is a positive stepsize. Otherwise, we stop the process and return $\bbx_t$ as an $(\eps,\gamma)$-second order stationary point of Problem~\eqref{eq:main_problem}. To check this claim, note that Algorithm \ref{alg_generic_framework} stops if we reach a point $\bbx_t$ that satisfies the first-order stationary condition $\nabla f(\bbx_t)^\top(\bbx-\bbx_t) \geq -\eps$, and the objective function value for the $\rho$-approximate solution of the quadratic subproblem is larger than $-\rho\gamma$, i.e., $q(\bbu_t)\geq -\rho \gamma $. The second condition alongside with the fact that $q(\bbu_t)$ satisfies \eqref{eq:rho_approx_solution} implies that $q(\bbu^*)\geq -\gamma$.  Therefore, for any $\bbx\in\ccalC$ and $\nabla f(\bbx_t)^\top(\bbx-\bbx_t) =0$, it holds that 
\begin{equation}\label{eq_SONOC}
 (\bbx-\bbx_t)^\top\nabla^2 f(\bbx_t)(\bbx-\bbx_t)\geq -\gamma.
\end{equation}
%%%
These two observations show that the outcome of the proposed framework in Algorithm \ref{alg_generic_framework} is an $(\eps,\gamma)$-SOSP of Problem~\eqref{eq:main_problem}. Now it remains to characterize the number of iterations that Algorithm~\ref{alg_generic_framework} needs to perform before reaching an $(\eps,\gamma)$-SOSP which we formally state in the following theorem.

%%%%%%%%%%%%%%%%%%%%%%%%%%%%%%%%%%%
%%%%%%%%%%%%%%%%%%%%%%%%%%%%%%%%%%%
%%%   A   L   G   O   R   I   T   H   M    %%%%%%%%%%%%%%%%
%%%%%%%%%%%%%%%%%%%%%%%%%%%%%%%%%%%
%%%%%%%%%%%%%%%%%%%%%%%%%%%%%%%%%%%
\begin{algorithm}[t]
\begin{algorithmic}[1]
\caption{Generic framework for escaping saddles in constrained optimization}\label{alg_generic_framework} 
\small{\REQUIRE Stepsize $\sigma>0$. Initialize $\bbx_0\in\ccalC$
\FOR {$t=1,2,\ldots$}
    \IF{  $\bbx_t$ is not an $\eps$-first order stationary point }
            \STATE Compute $\bbx_{t+1}$ using first-order information (Frank-Wolfe or projected gradient descent)
            \ELSE
            \STATE Find $ \bbu_t$: a $\rho$-approximate solution of \eqref{eq:quad_problem}
            \IF {$q(\bbu_t)  < -\rho\gamma$ }
            	\STATE  Compute the updated variable $\bbx_{t+1} = (1-\sigma) \bbx_{t} + \sigma \bbu_t$;\vspace{-1mm}
            \ELSE
            	\STATE Return $\bbx_t$ and stop. 
            \ENDIF
        \ENDIF
\ENDFOR}
\end{algorithmic}
\end{algorithm}

%%%%%%%%%%%%%%%%%%%%%%%%%%%%%%%%%%%
%%%%%%%%%%%%%%%%%%%%%%%%%%%%%%%%%%%
%%%%%    T  H  E  O  R  E  M     %%%%%%%%%%%%%%%%%%
%%%%%%%%%%%%%%%%%%%%%%%%%%%%%%%%%%%
%%%%%%%%%%%%%%%%%%%%%%%%%%%%%%%%%%%
\begin{theorem}\label{thm:main_thm}
Consider the optimization problem in \eqref{eq:main_problem}. Suppose the conditions in Assumptions \ref{assumption:lip_grad}-\ref{assumption:bounded_diameter} are satisfied. %If $\ccalT$ is the computational complexity of finding a $\rho$-approximate solution of the quadratic program in \eqref{eq:quad_problem}, then 
If in the first-order stage, i.e., Steps 2-4, we use the update of Frank-Wolfe or projected gradient descent, the generic framework proposed in Algorithm \ref{alg_generic_framework} finds an $(\eps,\gamma)$-second-order stationary point of Problem~\eqref{eq:main_problem} after at most $\mathcal{O}(\max\{\eps^{-2} , \rho^{-3}\gamma^{-3}\})$ iterations.
\end{theorem}

%The result in Theorem~\ref{thm:main_thm} shows that if the convex constraint $\ccalC$ is such that one can solve the quadratic subproblem in \eqref{eq:quad_problem} $\rho$-approximately, then the proposed generic framework finds an $(\eps,\gamma)$-SOSP point of Problem~\eqref{eq:main_problem} after at most $\mathcal{O}(\eps^{-2})$ first-order and $\mathcal{O}(\rho^{-3}\gamma^{-3})$ second-order updates. 

To prove the claim in Theorem~\ref{thm:main_thm}, we first review first-order conditional gradient and projected gradient algorithms and show that if the current iterate is not a first-order stationary point, by following either of these updates the objective function value decreases by a constant of $\mathcal{O}(\eps^2)$ (Section~\ref{sec:first_order_section}). We then focus on the second stage of Algorithm~\ref{alg_generic_framework} which corresponds to the case that the current iterate is an $\eps$-FOSP and we need to solve the quadratic program in \eqref{eq:quad_problem} approximately (Section~\ref{sec:second_stage}). In this case, we show that if the iterate is not an $(\eps,\gamma)$-SOSP, by following the update in \eqref{eq:second_stage_update} the objective function value decreases at least by a constant of $\mathcal{O}(\rho^3\gamma^3)$. Finally, by combining these two results it can be shown that Algorithm~\ref{alg_generic_framework} finds an $(\eps,\gamma)$-SOSP after at most $\mathcal{O}(\max\{\eps^{-2} , \rho^{-3}\gamma^{-3}\})$ iterations.

\section{First-Order Step: Convergence to a First-Order Stationary Point}\label{sec:first_order_section}

In this section, we study two different first-order methods for the first stage of Algorithm \ref{alg_generic_framework}. The result in this section can also be independently used  for convergence to an FOSP of Problem~\eqref{eq:main_problem} satisfying
%%%
\begin{equation}\label{eq:first_order_stationary_point_def}
\nabla f(\bbx^*)^\top(\bbx-\bbx^*)\geq -\eps,\qquad  \forall\ \bbx\in\ccalC, 
\end{equation}
%%%
where $\eps>0$ is a positive constant. Although for Algorithm~\ref{alg_generic_framework} we assume that $\ccalC$ has a specific structure as mentioned in \eqref{eq:quad_problem_template}, the results in this section hold for any closed and compact convex set~$\ccalC$. 
%Depending on the convex set $\ccalC$, one may prefer projection based method over conditional gradient methods or vice versa. %To be more precise, if the convex set $\ccalC$ is such that the cost of projection onto the convex set $\ccalC$ is lower than solving a linear program, e.g., when $\ccalC$ is sphere in $\reals^d$, the projected gradient-based methods are preferable to conditional gradient methods such as the Frank-Wolfe algorithm. Indeed, if the opposite holds and solving a linear program is computationally cheaper than projection, e.g., when $\ccalC$ is a polyhedral, conditional gradient methods are preferable. 
To keep our result as general as possible, in this section, we study both conditional gradient and projected-based methods when they are used in the first-stage of the proposed generic framework. %As a consequence of this analysis, we also report the overall computational complexity of these algorithms to reach a first-order stationary point of Problem~\ref{eq:main_problem} as defined in \eqref{eq:first_order_stationary_point_def}.

%%%%%%%%%%%%%%%%%%%%%%%%%%%%%%%%%%
%%%%%%%%%%%%%%%%%%%%%%%%%%%%%%%%%%
%%%%%  S  U  B  -- S  E  C  T  I  O  N    %%%%%%%%%%%%%
%%%%%%%%%%%%%%%%%%%%%%%%%%%%%%%%%%
%%%%%%%%%%%%%%%%%%%%%%%%%%%%%%%%%%

\subsection{Conditional gradient update}

The conditional gradient (Frank-Wolfe) update has two steps. We first solve the linear program 
%%%%
\begin{equation}\label{eq:FW_first_step}
\bbv_t=\argmax_{\bbv\in\ccalC} \{-\nabla f(\bbx_t)^\top\bbv\}.
\end{equation}
%%%%
%which finds the vector $\bbv_t\in\ccalC$ that maximizes the inner product with the negative of the current gradient $-\nabla f(\bbx_t)$. 
Then, we compute the updated variable $\bbx_{t+1}$ according to the update
%%%%
\begin{equation}\label{eq:FW_second_step}
\bbx_{t+1}= (1-\eta)\bbx_t + \eta \bbv_t,
\end{equation}
%%%%
where $\eta$ is a stepsize. %This update shows that the iterate $\bbx_{t+1}$ is a convex combination of the previous iterate $\bbx_t$ and the vector $\bbv_t$. 
%If we start from a feasible point $\bbx_0\in\ccalC$ it follows by induction that all the iterates $\bbx_t$ for $t\geq0$ are in the feasible set $\ccalC$.
In the following proposition, we show that if the current iterate is not an $\eps$-first order stationary point, then by updating the variable according to \eqref{eq:FW_first_step}-\eqref{eq:FW_second_step} the objective function value decreases. %Further, the computational complexity of the conditional gradient method defined in \eqref{eq:FW_first_step}-\eqref{eq:FW_second_step} to reach a first-order stationary point is studied in the following proposition. 
The proof of the following proposition is adopted from \citep{lacoste2016convergence}.%, but since in \eqref{eq:FW_second_step} the stepsize is fixed the proposed analysis is slightly different from the one in \citep{lacoste2016convergence} which either requires a line-search or access to an upper bound for the curvature constant of function $f$ to set the stepsize.

%%%%%%%%%%%%%%%%%%%%%%%%%%%%%%%%%%%
%%%%%%%%%%%%%%%%%%%%%%%%%%%%%%%%%%%
%%%%%. P  R  O  P  O  S  I  T  I  O  N    %%%%%%%%%%%%%%
%%%%%%%%%%%%%%%%%%%%%%%%%%%%%%%%%%%
%%%%%%%%%%%%%%%%%%%%%%%%%%%%%%%%%%%
\begin{proposition}\label{FW_proposition}
Consider the optimization problem in \eqref{eq:main_problem}. Suppose Assumptions~\ref{assumption:lip_grad} and \ref{assumption:bounded_diameter} hold. Set the stepsize in \eqref{eq:FW_second_step} to $\eta=\eps/D^2L$. Then, if the iterate $\bbx_t$ at step $t$ is not an $\eps$-first order stationary point, the objective function value at the updated variable $\bbx_{t+1}$ satisfies the inequality 
%%%%
\begin{equation}
f(\bbx_{t+1}) \leq  f(\bbx_t) -  \frac{\eps^2}{2D^2L}. 
\end{equation}
%%%%
%Moreover, the sequence of iterates $\bbx_t$ reaches a first-order stationary point after at most $2D^2 L(f(\bbx_0)-f(\bbx^*))/\eps^2$ iterations, where $\bbx^*$ is a global maximizer of Problem \eqref{eq:main_problem}.
\end{proposition}
%%%

The result in Proposition \ref{FW_proposition} shows that by following the update of the conditional gradient method the objective function value decreases by $\mathcal{O}(\eps^2)$, if an $\eps$-FOSP is not achieved. %As a consequence of this result, the FW algorithm reaches an $\eps$-first-order stationary point after at most $\mathcal{O}(\eps^{-2})$ iterations, or equivalently, after at most solving $\mathcal{O}(\eps^{-2})$ linear programs.  

\begin{remark}
In step 3 of Algorithm~\ref{alg_generic_framework} we first check if $\bbx_t$ is an $\eps$-FOSP. This can be done by evaluating 
\begin{equation}\label{stop_criteria}
\min_{\bbx\in\ccalC} \{\nabla f(\bbx_t)^\top(\bbx-\bbx_t)\}=\max_{\bbx\in\ccalC} \{-\nabla f(\bbx_t)^\top\bbx\} + \nabla f(\bbx_t)^\top\bbx_t
\end{equation}
 and comparing the optimal value with $-\eps$. Note that the linear program in \eqref{stop_criteria} is the same as the one in \eqref{eq:FW_first_step}. Therefore, by checking the first-order optimality condition of $\bbx_t$, the variable $\bbv_t$ is already computed, and we need to solve only one linear program per iteration.
\end{remark}

%%%%%%%%%%%%%%%%%%%%%%%%%%%%%%%%%%
%%%%%%%%%%%%%%%%%%%%%%%%%%%%%%%%%%
%%%%%  S  U  B  -- S  E  C  T  I  O  N    %%%%%%%%%%%%%
%%%%%%%%%%%%%%%%%%%%%%%%%%%%%%%%%%
%%%%%%%%%%%%%%%%%%%%%%%%%%%%%%%%%%
\subsection{Projected gradient update}

The projected gradient descent (PGD) update consists of two steps: (i) descending through the gradient direction and (ii) projecting the updated variable onto the convex constraint set. These two steps can be combined together and the update can be explicitly written as 
\vspace{1mm}
\begin{equation}\label{eq:PGD_update}
\bbx_{t+1}= \pi_{\ccalC}\{\bbx_t -\eta \nabla f(\bbx_t)\},
\end{equation}
%%%
where $\pi_{\ccalC}(.)$ is the Euclidean projection onto the convex set $\ccalC$ and $\eta$ is a positive stepsize. In the following proposition, we first show that by following the update of PGD the objective function value decreases by a constant until we reach an $\eps$- FOSP. Further, we show that the number of required iterations for PGD to reach an $\eps$-FOSP is of $\mathcal{O}(\eps^{-2})$.

%%%%%%%%%%%%%%%%%%%%%%%%%%%%%%%%%%%
%%%%%%%%%%%%%%%%%%%%%%%%%%%%%%%%%%%
%%%%%. P  R  O  P  O  S  I  T  I  O  N    %%%%%%%%%%%%%%
%%%%%%%%%%%%%%%%%%%%%%%%%%%%%%%%%%%
%%%%%%%%%%%%%%%%%%%%%%%%%%%%%%%%%%%
\begin{proposition}\label{PGD_proposition}
Consider Problem \eqref{eq:main_problem}. Suppose  Assumptions~\ref{assumption:lip_grad} and \ref{assumption:bounded_diameter} are satisfied. Further, assume that the gradients $\nabla f(\bbx)$ are uniformly bounded by $K$ for all $\bbx\in\ccalC$. If the stepsize of the projected gradient descent method defined in \eqref{eq:PGD_update} is set to $\eta=1/L$ the objective function value decreases by
%%%%
\begin{equation}\label{pgd_dec}
f(\bbx_{t+1}) \leq  f(\bbx_t)  -\frac{\eps^2L}{2(K+ LD)^2},
\end{equation}
%%%%
Moreover, iterates reach a first-order stationary point satisfying \eqref{eq:first_order_stationary_point_def} after at most $\mathcal{O}(\eps^{-2})$ iterations.
\end{proposition}

Proposition \ref{PGD_proposition} shows that by following the update of PGD the function value decreases by $\mathcal{O}(\eps^2)$ until we reach an $\eps$-FOSP. It further shows PGD obtains an $\eps$-FOSP satisfying \eqref{eq:first_order_stationary_point_def} after at most $\mathcal{O}(\eps^{-2})$ iterations. To the best of our knowledge, this result is also novel, since the only convergence guarantee for PGD in \citep{ghadimi2016mini} is in terms of number of iterations to reach a point with a gradient mapping norm less than $\eps$, while our result characterizes number of iterations to satisfy \eqref{eq:first_order_stationary_point_def}.

\begin{remark}
To use the PGD update in the first stage of Algorithm~\ref{alg_generic_framework} one needs to define a criteria to check if $\bbx_t$ is an $\eps$-FOSP or not. However, in PGD we do not solve the linear program $\min_{\bbx\in\ccalC} \{\nabla f(\bbx_t)^\top(\bbx-\bbx_t)\}$. This issue can be resolved by checking the condition $\|\bbx_t- \bbx_{t+1}\| \leq \eps/(K+ LD)$ which is a sufficient condition for the condition in \eqref{eq:first_order_stationary_point_def}. In other words, if this condition holds we stop and $\bbx_t$ is an $\eps$-FOSP; otherwise, the result in \eqref{pgd_dec} holds and the function value decreases. For more details please check the proof of Proposition \ref{PGD_proposition}.
%\red{we should swap step 2 and step 3 in this case, since we first perform the update then check first-order optimality.}
\end{remark}
\section{Second-Order Step: Escape from Saddle Points}\label{sec:second_stage}

In this section, we study the second stage of the framework in Algorithm~\ref{alg_generic_framework} which corresponds to the case that the current iterate is an $\eps$-FOSP. Note that when we reach a critical point the goal is to find a feasible point $\bbu\in \ccalC $ in the tangent space $\nabla f (\bbx_t)^\top(\bbu-\bbx_t)=0$ that makes the inner product $(\bbu-\bbx_t)^\top \nabla^2 f(\bbx_t)(\bbu-\bbx_t)$ smaller than $-\gamma$. To achieve this goal we need to check the minimum value of this inner product over the constraints, i.e., we need to solve the quadratic program in \eqref{eq:quad_problem} up to a constant factor $\rho\in(0,1]$. In the following proposition, we show that the updated variable according to \eqref{eq:second_stage_update} decreases the objective function value if the condition $q(\bbu_t)< -\rho \gamma $ holds.
%%%%
%\begin{align}\label{eq:quad_problem}
%& \min_{\bbu} \quad  (\bbu-\bbx_t)^\top \nabla^2 f(\bbx_t)(\bbu-\bbx_t)\nonumber\\
%& \text{s.t.} \quad \bbu\in \ccalC,\ \nabla f (\bbx_t)^\top(\bbu-\bbx_t)=0.
%\end{align}
%%%%

%%%%%%%%%%%%%%%%%%%%%%%%%%%%%%%%%%%
%%%%%%%%%%%%%%%%%%%%%%%%%%%%%%%%%%%
%%%%%. P  R  O  P  O  S  I  T  I  O  N    %%%%%%%%%%%%%%
%%%%%%%%%%%%%%%%%%%%%%%%%%%%%%%%%%%
%%%%%%%%%%%%%%%%%%%%%%%%%%%%%%%%%%%
\begin{proposition}\label{SOU_proposition}
Consider the quadratic program in \eqref{eq:quad_problem}. Let $\bbu_t$ be a $\rho$-approximate solution for quadratic subproblem in \eqref{eq:quad_problem}. Suppose that Assumptions \ref{assumption:lip_hessian} and \ref{assumption:bounded_diameter} hold. Further, set the stepsize $\sigma=\rho\gamma/MD^3$. If the quadratic objective function value $q$ evaluated at $\bbu_t$ satisfies the condition $q(\bbu_t)< -\rho \gamma $, then the updated variable according to \eqref{eq:second_stage_update} satisfies the inequality 
%%%%
\begin{align}
 f(\bbx_{t+1})   \leq   f(\bbx_t) -\frac{\rho^3\gamma^3}{3M^2D^6}.
 \end{align}
%%%%
\end{proposition}

%The result in Proposition \ref{SOU_proposition} shows that if $\bbu_t$ which is a $\rho$-approximate solution of the quadratic program in \eqref{eq:quad_problem} satisfies the inequality $q(\bbu_t)< -\rho \gamma $, then by following the direction $\bbu_t-\bbx_t$ with a proper stepsize the objective function value decreases by a constant of $\mathcal{O}(\rho^3\gamma^3)$. 

%Moreover, this result shows that after at most $\tau=\mathcal{O}((f(\bbx_0)-f(\bbx^*)) \rho^{-3}\gamma^{-3})$ iterations the algorithm reaches a $\rho$-approximate solution point satisfying $q(\bbu_\tau)\geq -\rho \gamma$, which follows that $\bbx_{\tau}$ satisfies the second order necessary optimality condition in \eqref{eq_SONOC}.

The only unanswered question is how to solve the quadratic subproblem in \eqref{eq:quad_problem} up to a constant factor $\rho\in(0,1]$. For general $\ccalC$, the quadratic subproblem could be NP-hard \citep{murty1987some}; however, for some special choices of the convex constraint $\ccalC$, this quadratic program (QP) can be solved either exactly or approximately up to a constant factor. In the following section, we focus on the quadratic constraint case, but indeed there are other classes of constraints that satisfy our required condition.

\subsection{Quadratic constraints case}

In this section, we focus on the case where the constraint set $\ccalC$ is defined as the intersection of $m$ ellipsoids centered at the origin.\footnote{To simplify the constant factor approximation $\rho$ we assume ellipsoids are centered at the origin. If we drop this assumption then $\rho$ will depend on the maximum distance between the origin and the boundary of each of the ellipsoids, e.g., see equation (6) in  \citep{tseng2003further}.} In particular, assume that the set  $\ccalC$ is given by
%%%
\begin{equation}
\ccalC:= \{\bbx \in \reals^d \mid \bbx^\top\bbQ_i\bbx  \leq 1, \ \  \forall\ i=1,\dots,m\},
\end{equation}
 %%%
where $\bbQ_i\in \mathbb{S}_{+}^d$. Under this assumption, the QP in \eqref{eq:quad_problem} can be written as
%%%
\begin{align}\label{main_qcqp_problem}
& \min_{\bbu} \quad  (\bbu-\bbx_t)^\top \nabla^2 f(\bbx_t)(\bbu-\bbx_t)\nonumber\\
& \text{s.t.} \quad \bbu^\top\bbQ_i\bbu  \leq 1,
\quad \for\ i=1,\dots, m \ \quad \text{and}\ \ \nabla f (\bbx_t)^\top(\bbu-\bbx_t)=0. 
\end{align}
%%%
Note that the equality constraint $\nabla f (\bbx_t)^\top(\bbu-\bbx_t)=0$ does not change the hardness of the problem and can be easily eliminated. To do so, first define a new optimization variable $\bbz:= \bbu-\bbx_t$ to obtain  
%%%
\begin{align}\label{main_qcqp_problem_2}
& \min_{\bbz} \quad  \bbz^\top \nabla^2 f(\bbx_t)\bbz\nonumber\\
& \text{s.t.} \quad (\bbz+\bbx_t)^\top\bbQ_i(\bbz+\bbx_t) \leq 1,\quad \for\ i=1,\dots, m \ \quad\ \text{and}\ \ \nabla f (\bbx_t)^\top\bbz=0,
\end{align}
%%%
%where $\tbs_i= \bbs_i +2 \bbQ_i\bbx_t$ and $\tdr_i = r_i+ \bbx_t^\top\bbs_i$.
Then, find a basis for the tangent space $\nabla f(\bbx_t)^\top\bbz=0$. Indeed, using the Gramm-Schmidt procedure, we can find an orthonormal basis for the space $\reals^d$ of the form $\{\bbv_1,\dots, \bbv_{d-1}, \frac{\nabla f(\bbx_t)}{\|\nabla f(\bbx_t)\|}\}$ at the complexity of $\mathcal{O}(d^3)$. If we define $\bbA=[\bbv_1;\dots;\bbv_{d-1}]\in \reals^{d\times d-1}$ as the concatenation of the vectors $\{\bbv_1,\dots, \bbv_{d-1}\}$, then any vector $\bbz$ satisfying $\nabla f(\bbx_t)^\top\bbz=0$ can be written as $\bbz=\bbA\bby$ where $\bby\in \reals^{d-1}$. Hence, \eqref{main_qcqp_problem_2} is equivalent to
%%%
\begin{align}\label{main_qcqp_problem_3}
 &\min_{\bbz} \quad  \bby^\top\bbA^\top \nabla^2 f(\bbx_t)\bbA\bby\nonumber\\
 &\text{s.t.} \quad (\bbA\bby+\bbx_t)^\top\bbQ_i(\bbA\bby+\bbx_t) \leq 1,\quad \for\ i=1,\dots, m.
\end{align}
%%%
This procedure reduces the dimension of the problem from $d$ to $d-1$.
It is not hard to check that the center of ellipsoids in \eqref{main_qcqp_problem_3} is $-\bbA^\top \bbx_t$. By a simple change of variable $\bbA\hby:= \bbA\bby+\bbx_t$ we obtain 
%%%
\begin{align}\label{main_qcqp_problem_4}
& \min_{\bbz}
  \quad  \hby^\top\bbA^\top \nabla^2 f(\bbx_t)\bbA \hby
   - 2 \bbx_t^\top \nabla^2 f(\bbx_t) \bbA \hby
    +\bbx_t^\top \nabla^2 f(\bbx_t)\bbx_t 
\nonumber\\
 &\text{s.t.} \quad 
 \hby^\top \bbA^\top\bbQ_i\bbA\hby \leq 1,
 \quad \for\ i=1,\dots, m.
\end{align}
%%%
Define the matrices $\tbQ_i:=\bbA^\top\bbQ_i\bbA$ and $\bbB_t:=\bbA^\top \nabla^2 f(\bbx_t)\bbA$, the vector $\bbs_t=- 2\bbx_t^\top \nabla^2 f(\bbx_t) \bbA$, and the scalar $\bbc_t:=\bbx_t^\top \nabla^2 f(\bbx_t)\bbx_t  $. Using these definitions  the problem reduces to 
%%%%
%%%
\begin{align}\label{main_qcqp_problem_5}
& \min_{\bbz}
  \quad  q(\hby) := \hby^\top\bbB_t\hby
   +\bbs_t^\top \hby+c_t
\qquad 
 &\text{s.t.} \quad 
 \hby^\top\tbQ_i\hby \leq 1,
 \quad \for\ i=1,\dots, m.
\end{align}
%%%

%\begin{align}\label{main_qcqp_problem_4}
% \min_{\bbz} \quad  \bby^\top\bbB_t\bby\qquad \text{s.t.} \quad \bby^\top\tbQ_i\bby + \hbs_i^\top\bby +\tdr_i\leq 0,\quad \for\ i=1,\dots, m .
%\end{align}
%%%%
Note that the matrices $\tbQ_i\in \mathbb{S}_{+}^d$ are positive semidefinite, while the matrix $\bbB_t \in \mathbb{S}^d$ might be indefinite. Indeed, the optimal objective function value of the program in \eqref{main_qcqp_problem_5}  is equal to the optimal objective function value of \eqref{main_qcqp_problem}. Further, note that if we find a $\rho$-approximate solution $\hby^*$ for \eqref{main_qcqp_problem_5}, we can recover a $\rho$-approximate solution $\bbu^*$ for \eqref{main_qcqp_problem} using the transformation $\bbu^* = \bbA\hby^*$.

The program in \eqref{main_qcqp_problem_5} is a specific \textit{Quadratic Constraint Quadratic Program} (QCQP), where all the constraints are centered at $\bb0$. For the specific case of $m=1$, the duality gap of this problem is zero and simply by transferring the problem to the dual domain one can solve  Problem \eqref{main_qcqp_problem_5} exactly. In the following proposition, we focus on the general case of $m\geq1$ and explain how to find a $\rho$-approximate solution for \eqref{main_qcqp_problem_5}.

\begin{proposition}\label{prop:quad_constraint}
Consider Problem \eqref{main_qcqp_problem_5} and define $q_{min}$ as the minimum objective value of the problem. Based on the result in \citep{fu1998approximation}, there exists a polynomial time method that obtains a point $\hby^*$
\begin{equation}\label{ye_bound}
 q(\hby^*) \leq \frac{1-\zeta}{m^2(1+\zeta)^2}\ q_{min} + \left(1-\frac{1-\zeta}{m^2(1+\zeta)^2}\right) \bbx_t^\top \nabla^2 f(\bbx_t)\bbx_t 
\end{equation}
after at most $\mathcal{O}( d^3 ( m\log(1/\delta) + \log(1/\zeta) +\log d))$ arithmetic operations, where $\delta$ is the ratio of the radius of the largest inscribing sphere over that of the smallest circumscribing sphere of the feasible set. Further, based on \citep{tseng2003further}, using a SDP relaxation of \eqref{main_qcqp_problem_5} one can find a point  $\hby^*$ such that 
\begin{equation}\label{tseng_bound}
 q(\hby^*) \leq \frac{1}{m}\ q_{min} + \left(1-\frac{1}{m}\right) \bbx_t^\top \nabla^2 f(\bbx_t)\bbx_t. 
\end{equation}
\end{proposition}

\begin{proof}
{If we define the function $\tilde{q}$ as $\tilde{q}(\bbx):={q}(\bbx)-c_t$, using the approaches in \citep{fu1998approximation} and \citep{tseng2003further}, we can find a $\rho$ approximate solution for $\min_{\hby} \tilde{q}(\hby) $ subject to  $\hby^\top\tbQ_i\hby \leq 1$ for $i=1,\dots, m$. In other words, we can find a point $\hby^*$ such that  $\tilde{q}(\hby^*) \leq \rho \ \tilde{q}_{min}$ where $0<\rho<1$ and  $\tilde{q}_{min}$ is the minimum objective function value of $\tilde{q}$ over the constraint set which satisfies $\tilde{q}_{min} = q_{min}-c_t$. Replacing $\tilde{q}(\hby^*) $ and $\tilde{q}_{min}$ by their definitions and regrouping the terms imply that $\hby^*$ satisfies the condition $q(\hby^*) \leq \rho q_{min} + (1-\rho)c_t$. Replacing $\rho$ by $\frac{1-\zeta}{m^2(1+\zeta)^2}$ (which is the constant factor approximation shown in \citep{fu1998approximation}) leads to the claim in \eqref{ye_bound}, and substituting $\rho$ by $1/m$ (which is the approximation bound in \citep{tseng2003further}) implies the result in \eqref{tseng_bound}.}
\end{proof}

The result in Proposition \ref{prop:quad_constraint} indicates that if $\bbx_t^\top \nabla^2 f(\bbx_t)\bbx_t$ is non-positive, then one can find a $\rho$-approximate solution for Problem~\eqref{main_qcqp_problem_5} and consequently Problem~\eqref{main_qcqp_problem}. This condition is satisfied if we assume that $\max_{\bbx\in \ccalC} \bbx^\top \nabla^2 f(\bbx)\bbx\leq 0$. For instance, for a concave minimization problem over the convex set $\ccalC$ this condition is satisfied. In fact, it can be shown that our analysis still stands even if $\max_{\bbx\in \ccalC} \bbx^\top \nabla^2 f(\bbx)\bbx$ is at most $\mathcal{O}( \gamma) $.  Note that this condition is significantly weaker than requiring the function to be concave when restricted to the feasible set. The condition essentially implies that the quadratic term in the Taylor expansion of the function evaluated at the origin should be negative (or not too positive).

\begin{corollary}
Consider a convex set $\ccalC$ which is defined as the intersection of $m\geq 1$ ellipsoids centered at the origin. Further, assume that the objective function Hessian $\nabla^2 f$ satisfies the condition $\max_{\bbx\in \ccalC} \bbx^\top \nabla^2 f(\bbx)\bbx\leq 0$. Then, for $\rho=1/m$ and $\rho=1/m^2$, it is possible to find a $\rho$-approximate solution of Problem~\eqref{eq:quad_problem} in time polynomial in $m$ and $d$. 
\end{corollary}

By using the approach in \citep{fu1998approximation}, we can solve the QCQP in \eqref{main_qcqp_problem_4} with the approximation factor $\rho\approx 1/m^2$ for $m\geq 1$ at the overall complexity of $\mathcal{\tilde{O}}(md^3)$ when the constraint $\ccalC$ is defined as $m$ convex quadratic constraints. As the total number of calls to the second-order stage is at most $\mathcal{{O}}(\rho^{-3}\gamma^{-3})= \mathcal{{O}}(m^6\gamma^{-3})$, we obtain that the total number of arithmetic operations for the second-order stage is at most $\mathcal{\tilde{O}}(m^7d^3\gamma^{-3})$. The constant factor can be improved to $1/m$ if we solve the SDP relaxation problem suggested in \citep{tseng2003further}.

%The main idea of the proposed algorithm by \citep{fu1998approximation} is to approximate the feasible set by an inscribing ellipsoid and maximize the quadratic objective function over this ellipsoid to find a global minimizer $\bbx^*$, and then use $\bbx^*$ as an approximate global minimizer for the original QP. If by increasing the radius of the ellipsoid the enlarged ellipsoid contains the the feasible set, then using the result in \citep{ye1992affine} it follows that $\bbx^*$ is a good approximate solution of the original QCQP.  

%\begin{remark}
%If $m=1$, one can also use the S-Procedure \citep{polik2007survey} to solve \eqref{main_qcqp_problem_4} accurately. Further, \citep{nemirovski1999maximization} showed that if the sum of the matrices $\tbQ_i$ is positive definite and the constraints are homogeneous, i.e., $\hbs_i=0$, one can find a $\frac{1}{2\ln(2m\mu)}$-approximate solution of \eqref{main_qcqp_problem_4} by solving its SDP relaxation, where $\mu:= \min\{m, \max_i{ \rank(\tbQ_i)}\}$.
%\end{remark}

%%%%%%%%%%%%%%%%%%%%%%%%%%%%%%%%%%%%%%%%%%%%%%
%%%%%%%%%%%%%%%%%%%%%%%%%%%%%%%%%%%%%%%%%%%%%%
%%%%%%%%      S  E  C  T  I  O  N     %%%%%%%%%%%%%%%%%%%%%%%%%%
%%%%%%%%%%%%%%%%%%%%%%%%%%%%%%%%%%%%%%%%%%%%%%
%%%%%%%%%%%%%%%%%%%%%%%%%%%%%%%%%%%%%%%%%%%%%%

\section{Stochastic Extension}

In this section, we focus on stochastic constrained minimization problems. Consider the optimization problem in \eqref{eq:main_problem} when the objective function $f$ is defined as an expectation of a set of stochastic functions $F:\reals^{d}\times\reals^{r}\to \reals $ with inputs $\bbx\in \reals^d$ and $\bbTheta\in \reals^r$, where $\bbTheta$ is a random variable with probability distribution $\ccalP$. To be more precise, we consider the optimization problem
%%%%
\begin{equation}\label{eq:stoc_main_problem}
 \text{minimize}\ f(\bbx):=\E{F(\bbx,\bbTheta)}, \qquad \text{subject to}\ \bbx\in \ccalC.
\end{equation}
%%%
%where the expectation is with respect to the random variable $\bbTheta$, and $\ccalC$ is a closed convex set. 
Our goal is to find a point which satisfies the necessary optimality conditions with high probability. 

Consider the vector $\bbd_{t}= ({1}/{b_g})\sum_{i=1}^{b_g} \nabla F(\bbx_t,\bbtheta_{i})$ and matrix $\bbH_{t}= ({1}/{b_H})\sum_{i=1}^{b_H} \nabla^2 F(\bbx_t,\bbtheta_{i})$ as stochastic approximations of the gradient $\nabla f(\bbx_t)$ and Hessian $\nabla^2 f(\bbx_t)$, respectively.
Here $b_g$ and $b_H$ are the gradient and Hessian batch sizes, respectively, and the vectors $\bbtheta_i$ are the realizations of the random variable $\bbTheta$. In Algorithm~\ref{alg_generic_framework_stocastic}, we present the stochastic variant of our proposed scheme for finding an $(\eps,\gamma)$-SOSP of Problem~\eqref{eq:stoc_main_problem}. Algorithm~\ref{alg_generic_framework_stocastic} differs from Algorithm~\ref{alg_generic_framework} in using the stochastic gradients $\bbd_t$ and Hessians $\bbH_t$ in lieu of the exact gradients $\nabla f(\bbx_t)$ and $\nabla^2 f(\bbx_t)$ Hessians. The second major difference is the inequality constraint in step 6. Here instead of using the constraint $ \bbd_t^\top(\bbu-\bbx_t)=0$ we need to use  $ \bbd_t^\top(\bbu-\bbx_t)\leq r$, where $r>0$ is a properly chosen constant. This modification is needed to ensure that if a point satisfies this constraint with high probability it also satisfies the constraint $\nabla f(\bbx_t)^\top(\bbu-\bbx_t)=0$. {This modification implies that we need to handle a linear inequality constraint instead of the linear equality constraint, which is computationally manageable for some constraints including the case that $\ccalC$ is a single ball constraint \citep{jeyakumar2014trust}.} 
%but as we stated in Section~\ref{sec:second_stage} adding a linear or quadratic constraint is not challenging and only makes the approximation guarantee $\rho$ smaller}. 
To prove our main result we assume that the following conditions also hold. 

%%%%%%%%%%%%%%%%%%%%%%%%%%%%%%%%%%%
%%%%%%%%%%%%%%%%%%%%%%%%%%%%%%%%%%%
%%%   A   L   G   O   R   I   T   H   M    %%%%%%%%%%%%%%%%
%%%%%%%%%%%%%%%%%%%%%%%%%%%%%%%%%%%
%%%%%%%%%%%%%%%%%%%%%%%%%%%%%%%%%%%
\begin{algorithm}[t]
\begin{algorithmic}[1]
\caption{}\label{alg_generic_framework_stocastic} 
\small{\REQUIRE Stepsize $\sigma_t>0$. Initialize $\bbx_0\in\ccalC$
\FOR {$t=1,2,\ldots$}
\STATE Compute $\bbv_t=\argmax_{\bbv\in\ccalC} \{-\bbd_t^\top\bbv\}$
    \IF{  $\bbd_t^{\top}(\bbv_t-\bbx_t)< -\eps/2$ }
            \STATE Compute $\bbx_{t+1}= (1-\eta)\bbx_t +\eta \bbv_t $
            \ELSE
            \STATE Find $ \bbu_t$: a $\rho$-approximate solution of\\ $ \quad \min_{\bbu} \quad  (\bbu-\bbx_t)^\top \bbH_t(\bbu-\bbx_t)\qquad  \text{s.t.} \ \ \bbu\in\ccalC,\ \bbd_t^\top(\bbu-\bbx_t)\leq r.$
            \IF {$q(\bbu_t)  < -\rho\gamma/2$ }
            	\STATE  Compute the updated variable $\bbx_{t+1} = (1-\sigma) \bbx_{t} + \sigma \bbu_t$;
            \ELSE
            	\STATE Return $\bbx_t$ and stop. 
            \ENDIF
        \ENDIF
\ENDFOR}
\end{algorithmic}\end{algorithm}

%%%%%%%%%%%%%%%%%%%%%%%%%%%%%%%%%%
%%%%%%%%%%%%%%%%%%%%%%%%%%%%%%%%%%
%%%%%  D  E   F  I  N  I  T  I  O  N    %%%%%%%%%%%%%%%
%%%%%%%%%%%%%%%%%%%%%%%%%%%%%%%%%%
%%%%%%%%%%%%%%%%%%%%%%%%%%%%%%%%%%
\begin{assumption}\label{assumption:bounded_variance}
The variance of the stochastic gradients and Hessians are uniformly bounded by constants $\nu^2$ and $\xi^2$, respectively, i.e., for any $\bbx\in \ccalC$ and $\bbtheta$ we can write
%%%%
\begin{align}\label{eq:bound_on_var}
\E{\| \nabla F(\bbx,\bbtheta) - \nabla f(\bbx) \|^2} \leq  \nu^2, \qquad 
\E{\| \nabla^2 F(\bbx,\bbtheta)-\nabla^2 f(\bbx)\|^2} \leq  \xi^2.
\end{align}
%%%%%
\end{assumption}

%In the following theorem, we characterize the iteration complexity of Algorithm~\ref{alg_generic_framework_stocastic}.

%%%%%%%%%%%%%%%%%%%%%%%%%%%%%%%%%%%
%%%%%%%%%%%%%%%%%%%%%%%%%%%%%%%%%%%
%%%%%    T  H  E  O  R  E  M     %%%%%%%%%%%%%%%%%%
%%%%%%%%%%%%%%%%%%%%%%%%%%%%%%%%%%%
%%%%%%%%%%%%%%%%%%%%%%%%%%%%%%%%%%%
\begin{theorem}\label{thm:main_thm_stochastic}
Consider the optimization problem in \eqref{eq:stoc_main_problem}. Suppose the conditions in Assumptions \ref{assumption:lip_grad}-\ref{assumption:bounded_variance} are satisfied. If the batch sizes are $b_g=\mathcal{O}(\max\{\rho^{-4}\gamma^{-4}, \eps^{-2}\})$ and $b_H=\mathcal{O}(\rho^{-2}\gamma^{-2})$ and we set the parameter $r=\mathcal{O}(\rho^2\gamma^2)$, 
then the outcome of the proposed framework outlined in Algorithm~\ref{alg_generic_framework_stocastic} is an $(\eps,\gamma)$-second-order stationary point of Problem~\eqref{eq:stoc_main_problem} with high probability. Further, the total number of iterations to reach such point is at most $\mathcal{O}(\max\{\eps^{-2}, \rho^{-3}\gamma^{-3}\})$ with high probability. 
\end{theorem}

The result in Theorem~\ref{thm:main_thm_stochastic} indicates that the total number of iterations to reach an $(\eps,\gamma)$-SOSP is at most $\mathcal{O}(\max\{\eps^{-2}, \rho^{-3}\gamma^{-3}\})$. As each iteration at most requires $\mathcal{O}(\max\{\rho^{-4}\gamma^{-4}, \eps^{-2}\})$ stochastic gradient and $\mathcal{O}(\rho^{-2}\gamma^{-2})$ stochastic Hessian evaluations, the total number of stochastic gradient and Hessian computations to reach an $(\eps,\gamma)$-SOSP is of $\mathcal{O}(\max\{\eps^{-2}\rho^{-4}\gamma^{-4}, \eps^{-4}, \rho^{-7}\gamma^{-7}\})$ and $\mathcal{O}(\max\{\eps^{-2}\rho^{-3}\gamma^{-3},\rho^{-5}\gamma^{-5}\})$, respectively.

\section{Appendix}

%%%%%%%%%%%%%%%%%%%%%%%%%%%%%%%%%%%
%%%%%%%%%%%%%%%%%%%%%%%%%%%%%%%%%%%
%%%%%   P  R  O  O  F     %%%%%%%%%%%%%%%%%%%%%
%%%%%%%%%%%%%%%%%%%%%%%%%%%%%%%%%%%
%%%%%%%%%%%%%%%%%%%%%%%%%%%%%%%%%%%
\subsection{Proof of Proposition \ref{prop:nec_conds}}
The claim in \eqref{eq:nec_cond_first_order} follows from Proposition 2.1.2 in \citep{bertsekas1999nonlinear}. The proof for the claim in \eqref{eq:nec_cond_second_order} is similar to the proof of Proposition 2.1.2 in \citep{bertsekas1999nonlinear}, and we mention it for completeness. 

We prove the claim in \eqref{eq:nec_cond_second_order} by contradiction. Suppose that $(\bbx-\bbx^*)^\top\nabla^2 f(\bbx^*)(\bbx-\bbx^*) <  0$ for some $\bbx\in \ccalC$ satisfying $\nabla f(\bbx^*)^\top(\bbx-\bbx^*)= 0$. By the mean value theorem, for any $\eps>0$ there exists an $\alpha\in[0,1]$ such that 
%%%%
\begin{align}
&f(\bbx^*+\eps(\bbx-\bbx^*))\nonumber\\
& = f(\bbx^*) + \eps \nabla f(\bbx^*)^\top(\bbx-\bbx^*) + \eps^2 (\bbx-\bbx^*)\nabla^2 f(\bbx^*+\alpha\eps(\bbx-\bbx^*))^\top(\bbx-\bbx^*),
\end{align}
%%%%
Use the relation $\nabla f(\bbx^*)^\top(\bbx-\bbx^*)= 0$  to simplify the right hand side to
%%%%
\begin{equation}\label{eq:1}
f(\bbx^*+\eps(\bbx-\bbx^*)) =f(\bbx^*)+  \eps^2 (\bbx-\bbx^*)\nabla^2 f(\bbx^*+\alpha\eps(\bbx-\bbx^*))^\top(\bbx-\bbx^*).
\end{equation}
%%%%
Note that since $(\bbx-\bbx^*)^\top\nabla^2 f(\bbx^*)(\bbx-\bbx^*) <  0$ and the Hessian is continuous, we have for all sufficiently small $\eps>0$,  $(\bbx-\bbx^*)\nabla^2 f(\bbx^*+\alpha\eps(\bbx-\bbx^*))^\top(\bbx-\bbx^*)<0$. This observation and the expression in \eqref{eq:1} follows that for sufficiently small $\eps$ we have $f(\bbx^*+\eps(\bbx-\bbx^*))<f(\bbx^*)$. Note that the point $\bbx^*+\eps(\bbx-\bbx^*)$ for all $\eps\in [0,1]$ belongs to the set $\ccalC$ and satisfies the inequality $\nabla f(\bbx^*)^\top((\bbx^*+\eps(\bbx-\bbx^*))-\bbx^*)= 0$. Therefore, we obtained a contradiction of the local optimality of $\bbx^*$.

%%%%%%%%%%%%%%%%%%%%%%%%%%%%%%%%%%%
%%%%%%%%%%%%%%%%%%%%%%%%%%%%%%%%%%%
%%%%%   P  R  O  O  F     %%%%%%%%%%%%%%%%%%%%%
%%%%%%%%%%%%%%%%%%%%%%%%%%%%%%%%%%%
%%%%%%%%%%%%%%%%%%%%%%%%%%%%%%%%%%%
\subsection{Proof of Proposition \ref{FW_proposition}}
First consider the definition $G(\bbx_t)= \max_{\bbx \in \ccalC}\{ -\nabla f(\bbx_t)^\top(\bbx-\bbx_t) \}$ which is also known as Frank-Wolfe gap \citep{lacoste2016convergence}. This constant measures how close the point $\bbx_t$ is to be a first-order stationary point. If $G(\bbx_t)\leq\eps$, then $\bbx_t$ is an $\eps$-first-order stationary point. Let's assume that $G(\bbx_t)>\eps$. Then, based on the Lipschitz continuity of gradients and the definition of $G(\bbx_t)$ we can write
%%%%
\begin{align}
f(\bbx_{t+1})
&\leq  f(\bbx_t) +\nabla f(\bbx_t)^\top(\bbx_{t+1}-\bbx_{t}) +\frac{L}{2}\|\bbx_{t+1}-\bbx_{t}\|^2\nonumber\\
&=  f(\bbx_t) +\eta \nabla f(\bbx_t)^\top(\bbv_t-\bbx_{t}) +\frac{L\eta^2}{2}\|\bbv_{t}-\bbx_{t}\|^2\nonumber\\
&\leq  f(\bbx_t) -\eta G(\bbx_t) +\frac{\eta^2D^2L}{2},
\end{align}
%%%
where the last inequality follows from $\|\bbv_{t}-\bbx_{t}\|\leq D$. Replacing the stepsize $\eta$ by its value $\eps /D^2L$ and $G(\bbx_t)$ by its lower bound $\eps$ lead to  
%%%%
\begin{align}
f(\bbx_{t+1})
&\leq  f(\bbx_t) -\frac{\eps^2}{2D^2L} .
\end{align}
%%%
This result implies that if the current point $\bbx_t$ is not an $\eps$-first order stationary point, by following the update of Frank-Wolfe algorithm the objective function value decreases by ${\eps^2}/{2D^2L} $. 
Therefore, after at most $2D^2L(f(\bbx_0)-f(\bbx^*))/\eps^2 $ iterations we either reach the global minimum or one of the iterates $\bbx_t$ satisfies $G(\bbx_t)\leq\eps$ which implies that
%%%
\begin{equation}
\nabla f(\bbx_t)^\top(\bbx-\bbx_t)\geq -\eps,\qquad  \forall\ \bbx\in\ccalC,
\end{equation}
%%%
and the claim in Proposition \ref{FW_proposition} follows.

%%%%%%%%%%%%%%%%%%%%%%%%%%%%%%%%%%%
%%%%%%%%%%%%%%%%%%%%%%%%%%%%%%%%%%%
%%%%%   P  R  O  O  F     %%%%%%%%%%%%%%%%%%%%%
%%%%%%%%%%%%%%%%%%%%%%%%%%%%%%%%%%%
%%%%%%%%%%%%%%%%%%%%%%%%%%%%%%%%%%%
\subsection{Proof of Proposition \ref{PGD_proposition}}
First note, that based on the projection property we know that 
%%%
\begin{equation}
(\bbx_t -\eta \nabla f(\bbx_t) - \bbx_{t+1})^\top(\bbx- \bbx_{t+1})\leq 0, \qquad \forall\ \bbx\in \ccalC.
\end{equation}
%%%%
Therefore, by setting $\bbx=\bbx_t$ we obtain that
%%%
\begin{equation}
 \eta \nabla f(\bbx_t) ^\top(\bbx_{t+1}-\bbx_t)\leq -\|\bbx_t- \bbx_{t+1}\|^2.
\end{equation}
%%%%
Hence, we can replace the inner product $\nabla f(\bbx_t) ^\top(\bbx_{t+1}-\bbx_t)$ by its upper bound $-\|\bbx_t- \bbx_{t+1}\|^2/\eta$
%%%%
\begin{align}
f(\bbx_{t+1})
&\leq  f(\bbx_t) +\nabla f(\bbx_t)^\top(\bbx_{t+1}-\bbx_{t}) +\frac{L}{2}\|\bbx_{t+1}-\bbx_{t}\|^2\nonumber\\
&\leq f(\bbx_t) -\frac{\|\bbx_t- \bbx_{t+1}\|^2}{\eta}+\frac{L}{2}\|\bbx_{t+1}-\bbx_{t}\|^2\nonumber\\
&= f(\bbx_t) -\frac{L}{2}\|\bbx_{t+1}-\bbx_{t}\|^2,
\end{align}
%%%%
where the equality follows by setting $\eta=1/L$. Indeed, if $\bbx_{t+1}=\bbx_t$ then we are at a first-order stationary point, however, we need a finite time analysis. To do so, note that for any $\bbx\in\ccalC$ we have
%%%
\begin{equation}
(\bbx_t -\eta \nabla f(\bbx_t) - \bbx_{t+1})^\top(\bbx- \bbx_{t+1})\leq 0.
\end{equation}
 %%%
 Therefore, for any $\bbx\in\ccalC$  it holds
%%%
\begin{equation}
 \nabla f(\bbx_t) ^\top(\bbx- \bbx_{t+1})\geq  L (\bbx_t- \bbx_{t+1})^\top(\bbx- \bbx_{t+1}),
\end{equation}
 %%%
which implies that 
%%%
\begin{align}
 \nabla f(\bbx_t) ^\top(\bbx- \bbx_{t})
 &\geq  \nabla f(\bbx_t) ^\top( \bbx_{t+1}-\bbx_t) + L (\bbx_t- \bbx_{t+1})^\top(\bbx- \bbx_{t+1})\nonumber\\
 &\geq -K \|\bbx_{t+1}-\bbx_t\| - LD \|\bbx_t- \bbx_{t+1}\|\nonumber\\
  &\geq -(K+ LD) \|\bbx_t- \bbx_{t+1}\|,
\end{align}
 %%%
 where $K$ is an upper bound on the norm of gradient over the convex set $\ccalC$. Therefore, we can write 
 %%%
\begin{align}
\min_{\bbx\in\ccalC} \nabla f(\bbx_t) ^\top(\bbx- \bbx_{t})
\geq -(K+ LD) \|\bbx_t- \bbx_{t+1}\|,
\end{align}
 %%%
 Combining these results, we obtain that we should check the norm $\|\bbx_t- \bbx_{t+1}\|$ at each iteration and check whether if it is larger than $\eps/(K+ LD)$ or not. If the norm is larger than the threshold then
  %%%%
\begin{align}
 f(\bbx_{t+1})\leq  f(\bbx_t) -\frac{\eps^2L}{2(K+ LD)^2}.
\end{align}
%%%
If the norm is smaller than the threshold then we stop and the iterate $\bbx_t$ satisfies the inequality 
%%%
\begin{equation}
\nabla f(\bbx_t)^\top(\bbx-\bbx_t)\geq -\eps,\qquad  \forall\ \bbx\in\ccalC.
\end{equation}
%%%
Note that this process can not take more than $\mathcal{O}(\frac{f(\bbx_0)-f(\bbx^*)}{\eps^2})$ iterations.

%%%%%%%%%%%%%%%%%%%%%%%%%%%%%%%%%%%
%%%%%%%%%%%%%%%%%%%%%%%%%%%%%%%%%%%
%%%%%   P  R  O  O  F     %%%%%%%%%%%%%%%%%%%%%
%%%%%%%%%%%%%%%%%%%%%%%%%%%%%%%%%%%
%%%%%%%%%%%%%%%%%%%%%%%%%%%%%%%%%%%
\subsection{Proof of Proposition \ref{SOU_proposition}}
The Taylor's expansion of the function $f$ around the point $\bbx_t$ and $M$-Lipschitz continuity of the Hessians imply that
%%%
\begin{equation}\label{eq:proof_main_thm_100}
 f(\bbx_{t+1}) \leq  f(\bbx_t) + \nabla f(\bbx_t)^\top(\bbx_{t+1}-\bbx_t)  +\frac{1}{2} (\bbx_{t+1}-\bbx_t)^\top \nabla^2 f(\bbx) (\bbx_{t+1}-\bbx_t) + \frac{M}{6} \|\bbx_{t+1}-\bbx_t\|^3.
\end{equation}
%%%
Replace $\bbx_{t+1}-\bbx_t$ by the expression $\sigma(\bbu_{t}-\bbx_t)$ to obtain
%%%
\begin{equation}\label{eq:proof_main_thm_200}
 f(\bbx_{t+1})    \leq   f(\bbx_t) +\sigma \nabla f(\bbx_t)^\top( \bbu_{t}-\bbx_t)+ \frac{\sigma^2}{2} (\bbu_t-\bbx_t)^\top \nabla^2 f(\bbx) (\bbu_t-\bbx_t) + \frac{M\sigma^3}{6} \|\bbu_t-\bbx_t\|^3.
\end{equation}
%%%
Since, $\bbu_t$ is a $\rho$-approximate solution for the subproblem in \eqref{eq:quad_problem} with the objective function value $q(\bbu_t)\leq -\rho \gamma$, we can substitute the quadratic term $(\bbu_t-\bbx_t)^\top \nabla^2 f(\bbx) (\bbu_t-\bbx_t)$ by its upper bound $-\rho \gamma$. Additionally, the vector $\bbu_t$ is chosen such that  $\nabla f(\bbx_t)^\top( \bbu_{t}-\bbx_t)=0$ and therefore the linear term in \eqref{eq:proof_main_thm_200} can be eliminated. Further, the cubic term $ \|\bbu_t-\bbx_t\|^3$ is upper bounded by $D^3$ since both $\bbu_t$ and $\bbx_t$ belong to the convex set $\ccalC$. Applying these substitutions into \eqref{eq:proof_main_thm_200} yields
%%%
\begin{equation}\label{eq:proof_main_thm_300}
 f(\bbx_{t+1})    \leq   f(\bbx_t) -\frac{\sigma^2\rho\gamma}{2} + \frac{\sigma^3MD^3}{6} .
\end{equation}
%%%
By setting $\sigma=\rho\gamma/MD^3$ in \eqref{eq:proof_main_thm_300} it follows that
%%%
\begin{align}
 f(\bbx_{t+1})  &  \leq   f(\bbx_t) -\frac{\rho^3\gamma^3}{2M^2D^6} + \frac{\rho^3\gamma^3}{6M^2D^6} \nonumber\\
 &=  f(\bbx_t) -\frac{\rho^3\gamma^3}{3M^2D^6}.
\end{align}
%%%
Therefore, in this case, the objective function value decreases at least by a fixed value of $\mathcal{O}(\rho^{3}\gamma^{3})$. 
%%%%%%%%%%%%%%%%%%%%%%%%%%%%%%%%%%%
%%%%%%%%%%%%%%%%%%%%%%%%%%%%%%%%%%%
%%%%%   P  R  O  O  F     %%%%%%%%%%%%%%%%%%%%%
%%%%%%%%%%%%%%%%%%%%%%%%%%%%%%%%%%%
%%%%%%%%%%%%%%%%%%%%%%%%%%%%%%%%%%%
\subsection{Proof of Theorem \ref{thm:main_thm}}

Then at each iteration, either the first oder optimality condition is not satisfied and the function value decreases by a constant of  $\mathcal{O}(\eps^{2})$, or this condition is satisfied and we use a second-order update which leads to a objective function value decrement of $\mathcal{O}(\rho^{3}\gamma^{3})$. This shows that if have not reached an $(\eps,\gamma)$-second order stationary point the objective function value decreases at least by $\mathcal{O}(\min\{\eps^{2}, \rho^{3}\gamma^{3}\})$. Therefore, we either reach the global minimum or converge to an $(\eps,\gamma)$-second order stationary point of Problem~\eqref{eq:main_problem} after at most $\mathcal{O}\left(\frac{f(\bbx_0)-f(\bbx^*)}{\min\{\eps^{2}, \rho^{3}\gamma^{3}\}}\right)$ iterations which also can be written as $\mathcal{O}((f(\bbx_0)-f(\bbx^*))(\eps^{-2}+ \rho^{-3}\gamma^{-3}))$.

%%%%%%%%%%%%%%%%%%%%%%%%%%%%%%%%%%%
%%%%%%%%%%%%%%%%%%%%%%%%%%%%%%%%%%%
%%%%%   P  R  O  O  F     %%%%%%%%%%%%%%%%%%%%%
%%%%%%%%%%%%%%%%%%%%%%%%%%%%%%%%%%%
%%%%%%%%%%%%%%%%%%%%%%%%%%%%%%%%%%%
\subsection{Proof of Theorem \ref{thm:main_thm_stochastic}}

In this proof, for notation convenience, we define $\eps'=\eps/2$ and $\gamma'=\gamma/2$.

First, note that the condition in Assumption~\ref{assumption:bounded_variance} and the fact that $\nabla F(\bbx,\bbtheta)$ and $\nabla^2 F(\bbx,\bbtheta)$ are the unbiased estimators of the gradient $\nabla f(\bbx)$ and Hessian $\nabla^2 f(\bbx)$ imply that the variance of the batch gradient $\bbd_t$ and the batch Hessian $\bbH_t$ approximations are upper bounded by 
%%%
\begin{equation}\label{eq:bound_on_batch_approx}
\E{ \|\bbd_t-\nabla f(\bbx_t)\|^2}\leq \frac{\nu^2}{b_g}, \qquad 
\E{ \|\bbH_t-\nabla^2 f(\bbx_t)\|^2}\leq \frac{\xi^2}{b_H}.
\end{equation}
%%%
Here we assume that $b_g$ and $b_H$ satisfy the following conditions, 
%%%
\begin{equation}\label{condition_on_batch_sizes}
b_g =\max\left\{ \frac{324\nu^2 M^2 D^8}{\rho^4\gamma'^4}  , \frac{16D^2\nu^2}{\eps'^2}   \right\} ,  \qquad b_H =\frac{81D^4\xi^2}{\rho^2\gamma'^2}  . 
\end{equation}
%%%
We further set the parameter $r$ as
%%%
\begin{equation}\label{def_r_para}
r= \frac{\rho^2\gamma'^2}{18MD^3}.
\end{equation}
%%%

Now we proceed to analyze the complexity of Algorithm 2. First, consider the case that the current iterate $\bbx_t$ satisfies the inequality $\bbd_t^{\top} (\bbv_t-\bbx_t) < -\eps'$ and therefore we perform the first-order update in step 4. In this case, we can show that
%%%
\begin{align}\label{proof_stochastic_100}
f(\bbx_{t+1}) 
& \leq f(\bbx_t) + \nabla  f(\bbx_t)^\top (\bbx_{t+1}-\bbx_{t}) +\frac{L}{2}\|\bbx_{t+1}-\bbx_{t}\|^2
\nonumber\\
& = f(\bbx_t) + \eta \nabla  f(\bbx_t)^\top (\bbv_{t}-\bbx_{t}) +\frac{\eta^2L}{2}\|\bbv_t-\bbx_{t}\|^2
\nonumber\\
& \leq  f(\bbx_t) + \eta \bbd_t^\top (\bbv_{t}-\bbx_{t}) +\eta (\nabla  f(\bbx_t)-\bbd_t)^\top (\bbv_{t}-\bbx_{t}) +\frac{\eta^2LD^2}{2}
\nonumber\\
& \leq  f(\bbx_t) - \eta\eps' +\eta D \|\nabla  f(\bbx_t)-\bbd_t\| +\frac{\eta^2LD^2}{2},
\end{align}
%%%
where in the last inequality we used $\bbd_t^{\top} (\bbv_t-\bbx_t) < -\eps'$ and the fact that both $\bbv_t$ and $\bbx_t$ belong to the set $\ccalC$ and therefore $\|\bbx_t-\bbv_t\|\leq D$. Consider $\ccalF_t$ as the sigma algebra that measures all sources of randomness up to step $t$. Then, computing the expected value of both sides of \eqref{proof_stochastic_100} given $\ccalF_t$ leads to
%%%
\begin{align}\label{proof_stochastic_200}
\E{f(\bbx_{t+1})\mid \ccalF_t } \leq  f(\bbx_t) - \eta\eps' +\frac{\eta D \nu}{\sqrt{b_g}} +\frac{\eta^2LD^2}{2}
\end{align}
%%%
where we used the inequality $\E{X}\leq \sqrt{\E{X^2}}$ when $X$ is a positive random variable. Replace the stepsize $\eta$ by its value ${\eps'}/({D^2L})$ and the batch size {$b_g$ by its lower bound ${(16D^2\nu^2)}/({\eps'^2})$} to obtain
%%%
\begin{align}\label{proof_stochastic_300}
\E{f(\bbx_{t+1}) \mid \ccalF_t } \leq  f(\bbx_t) -\frac{\eps'^2}{4D^2L} .
\end{align}
%%%
Hence, in this case, the objective function value decreases in expectation by a constant factor of $\mathcal{O}(\eps'^2)$.

Now we proceed to study the case that the current iterate $\bbx_t$ does not satisfy the inequality $\bbd_t^{\top} (\bbv_t-\bbx_t) < -\eps'$ and we need to perform the second-order update in step 8. 
%%%%%%%%%%%%
%Further note that 
%\begin{align}
%\E{\nabla f(\bbx_t)^\top(\bbx-\bbx_t)}\geq -\eps' cD+\E{\bbd_t^\top(\bbx-\bbx_t)}
% \end{align}
%%%%
%Hence, we should set $\eps'=\eps/(1+cD)$. This observation implies that, if $\bbd_t^{\top}(\bbv_t-\bbx_t)< -\eps'$ holds the objective function value decrease by $\mathcal{O}(\eps'^2)=\mathcal{O}(\eps^2)$ in expectation. If we reach a point $\bbx_t$ satisfying $\bbd_t^{\top}(\bbv_t-\bbx_t)\geq  -\eps'$, then in expectation it satisfies
%%%% 
%$$\E{\nabla f(\bbx_t)^\top(\bbx-\bbx_t)}\geq -\eps' cD-\eps' =-\eps $$
%%%%
%for all $\bbx\in \ccalC$. Note that indeed using the Markov's inequality one cane show that 
%%%%
%\begin{equation}
%P(\|\bbd_t-\nabla f(\bbx_t)\|\leq \delta )\leq 1-\frac{\nu_t^2}{\delta^2}
%\end{equation}
%%%%
%By repeating the steps, we obtain that with probability 
%%%%%%%%%%%%%%%
In this case, we can show that 
%%%
\begin{align}\label{proof_stochastic_400}
 f(\bbx_{t+1}) 
 & \leq  f(\bbx_t) + \nabla f(\bbx_t)^\top(\bbx_{t+1}-\bbx_t)  +\frac{1}{2} (\bbx_{t+1}-\bbx_t)^\top \nabla^2 f(\bbx) (\bbx_{t+1}-\bbx_t) + \frac{M}{6} \|\bbx_{t+1}-\bbx_t\|^3\nonumber\\
 & \leq  f(\bbx_t) + \sigma \nabla f(\bbx_t)^\top(\bbu_t-\bbx_t)  +\frac{\sigma^2}{2} (\bbu_t-\bbx_t)^\top \nabla^2 f(\bbx) (\bbu_t-\bbx_t) + \frac{\sigma^3 MD^3}{6}\nonumber\\
  & \leq  f(\bbx_t) + \sigma \bbd_t^\top(\bbu_t-\bbx_t)+ \sigma (\nabla f(\bbx_t)-\bbd_t)^\top(\bbu_t-\bbx_t)  +\frac{\sigma^2}{2} (\bbu_t-\bbx_t)^\top \bbH_t  (\bbu_t-\bbx_t) \nonumber\\
& \qquad + \frac{\sigma^2}{2} (\bbu_t-\bbx_t)^\top (\nabla^2 f(\bbx)-\bbH_t) (\bbu_t-\bbx_t) +\frac{\sigma^3 MD^3}{6}.
\end{align}
%%%
Note that $\bbu_t$ is a $\rho$-approximate solution for the subproblem in step 6 of Algorithm 2, with the objective function value less than $ -\rho \gamma'$. This observation implies that the quadratic term $(\bbu_t-\bbx_t)^\top \bbH_t  (\bbu_t-\bbx_t)$ is bounded above by $-\rho\gamma'$. Further, the linear term $ \bbd_t^\top(\bbu_t-\bbx_t) $ is less than $r$ according to the constraint of the subproblem. Applying these substitutions and using the Cauchy-Schwartz inequality multiple times lead to
%%%
\begin{align}\label{proof_stochastic_500}
 f(\bbx_{t+1}) 
 \leq  f(\bbx_t) +\sigma r+ \sigma D \|\bbd_t-\nabla f(\bbx_t)\|  -\frac{\sigma^2\rho\gamma'}{2} + \frac{\sigma^2D^2}{2} \| \bbH_t- \nabla^2 f(\bbx)\| +\frac{\sigma^3 MD^3}{6}.
\end{align}
%%%
Compute the conditional expected value of both sides of \eqref{proof_stochastic_500} and use the inequalities in \eqref{eq:bound_on_batch_approx} to obtain
%%%%
\begin{align}\label{proof_stochastic_600}
\E{ f(\bbx_{t+1}) \mid \ccalF_t}
 \leq  f(\bbx_t) +\sigma r+ \frac{\sigma D \nu}{\sqrt{b_g}}  -\frac{\sigma^2\rho\gamma'}{2} + \frac{\sigma^2D^2 \xi}{2\sqrt{b_H}} +\frac{\sigma^3 MD^3}{6}.
\end{align}
%%%
By setting the stepsize $\sigma=\rho\gamma'/MD^3$ in \eqref{proof_stochastic_600}  it follows that
%%%%
\begin{align}\label{proof_stochastic_700}
\E{ f(\bbx_{t+1}) \mid \ccalF_t}
 \leq  f(\bbx_t)  -\frac{\rho^3\gamma'^3}{3L^2D^6} +\frac{r\rho\gamma'}{MD^3} + \frac{\rho\gamma'  \nu}{MD^2\sqrt{b_g}}+ \frac{\rho^2\gamma'^2 \xi}{2M^2D^4\sqrt{b_H}} .
\end{align}
%%%
Moreover, setting {$r=\frac{\rho^2\gamma'^2}{18MD^3}$} and $b_H=\frac{81D^4\xi^2}{\rho^2\gamma'^2}$, and replacing $b_g$ by its lower bound $\frac{324\nu^2 M^2 D^8}{\rho^4\gamma'^4}$ lead to
%%%
\begin{align}\label{proof_stochastic_800}
\E{ f(\bbx_{t+1})\mid \ccalF_t}
 \leq  f(\bbx_t) -\frac{\rho^3\gamma'^3}{6M^2D^6}
\end{align}
%%%%
Hence, in this case, the expected objective function value decreases by a constant of $\mathcal{O}(\rho^3\gamma'^3)$. 

By combining the results in \eqref{proof_stochastic_300} and \eqref{proof_stochastic_800}, we obtain that if the iterate $\bbx_t$ is not the final iterate the objective function value at step $t+1$ satisfies the following ineqaulity
%%%
\begin{align}
\E{ f(\bbx_{t+1}) \mid \ccalF_t}
 \leq  f(\bbx_t) -\min\left\{\frac{\eps'^2}{4LD^2} , \frac{\rho^3\gamma'^3}{6M^2D^6} \right\}.
\end{align}
%%%
Let us define $T$ as the number of iterations we perform until Algorithm 2 stops. We use an argument similar to Wald's lemma to derive an upper bound on the expected number of iterations $T$ that we need to run the algorithm. Note that 
%%%
\begin{align}
\E{ f(\bbx_0)-f(\bbx_T) } 
& = \E{\sum_{t=1}^{T} (f(\bbx_{t-1}) -f(\bbx_{t}) )} \nonumber\\
& =  \E{\E{\sum_{t=1}^{T} (f(\bbx_{t-1}) -f(\bbx_{t}) )} \Bigg|\ T=k} \nonumber\\
& =  \E{\E{\sum_{t=1}^{k} (f(\bbx_{t-1}) -f(\bbx_{t}) )} \Bigg|\ T=k} \nonumber\\
& =  \sum_{k=1}^\infty \E{\sum_{t=1}^{k} (f(\bbx_{t-1}) -f(\bbx_{t}) )} \mathbb{P}(T=k) \nonumber\\
& =  \sum_{k=1}^\infty \sum_{t=1}^{k} \E{ (f(\bbx_{t-1}) -f(\bbx_{t}) )} \mathbb{P}(T=k) \nonumber\\
& \geq   \sum_{k=1}^\infty \sum_{t=1}^{k}  \min\left\{\frac{\eps'^2}{4LD^2} , \frac{\rho^3\gamma'^3}{6M^2D^6} \right\} \ \mathbb{P}(T=k) \nonumber\\
& =  \min\left\{\frac{\eps'^2}{4LD^2} , \frac{\rho^3\gamma'^3}{6M^2D^6} \right\} \sum_{k=1}^\infty  k\ \mathbb{P}(T=k) \nonumber\\
& =   \min\left\{\frac{\eps'^2}{4LD^2} , \frac{\rho^3\gamma'^3}{6M^2D^6} \right\} \E{T}.
\end{align}
%%%
Hence, $\E{T}\leq \E{f(\bbx_0)- f(\bbx_T) } /{\min\left\{\frac{\eps'^2}{4LD^2} , \frac{\rho^3\gamma'^3}{6M^2D^6} \right\}}$. We further know that $f(\bbx_T)\geq f(\bbx^*)$ which implies that 
\begin{equation}
\E{T}\leq (f(\bbx_0) -  f(\bbx^*))\max\left\{\frac{4LD^2}{\eps'^2} , \frac{6M^2D^6} {\rho^3\gamma'^3}\right\}.
\end{equation}
Using Markov's inequality we can show that 
\begin{align}
\mathbb{P}\left( T \leq a  \right)\geq 1-\frac{ (f(\bbx_0) -  f(\bbx^*))\max\left\{\frac{4LD^2}{\eps'^2} , \frac{6M^2D^6} {\rho^3\gamma'^3}\right\}}{a}
\end{align}
%%%
Set $a=\frac{(f(\bbx_0) -  f(\bbx^*))}{\delta}\max\left\{\frac{4LD^2}{\eps'^2} , \frac{6M^2D^6} {\rho^3\gamma'^3}\right\}$ to obtain that 
%%%
\begin{align}
\mathbb{P}\left( T \leq \frac{ (f(\bbx_0) -  f(\bbx^*))\max\left\{\frac{4LD^2}{\eps'^2} , \frac{6M^2D^6} {\rho^3\gamma'^3}\right\}}{\delta}  \right)\geq 1-\delta.
\end{align}
%%%
Therefore, it follows that with high probability the total number of iterations $T$ that Algorithm 2 runs is at most $\mathcal{O}(\max\left\{{\eps'^{-2}} ,{\rho^{-3}\gamma'^{-3}}\right\})$. 

%The objective function value after $T$ iterations satisfies 
%%%%
%\begin{equation}
%\E{f(\bbx_T)-f(\bbx^*)} \leq f(\bbx_0)-f(\bbx^*)-T\min(\mathcal{O}(\eps'^2),\mathcal{O}(\gamma'^3))
%\end{equation}
%Using Markov's inequality it follows that
%%%%%
%\begin{align}
%\mathbb{P}\left( f(\bbx_T)-f(\bbx^*)  \leq a  \right)\geq 1-\frac{ f(\bbx_0)-f(\bbx^*)-T\min(\mathcal{O}(\eps'^2),\mathcal{O}(\gamma'^3))}{a}
%\end{align}
%%%%
%If we set $a=( f(\bbx_0)-f(\bbx^*)-T\min(\mathcal{O}(\eps'^2),\mathcal{O}(\gamma'^3)))/\delta$, then with probability at least $1-\delta$ we have
%%%%
%\begin{equation}
%f(\bbx_T)-f(\bbx^*)  \leq \frac{f(\bbx_0)-f(\bbx^*)-T\min(\mathcal{O}(\eps'^2),\mathcal{O}(\gamma'^3)))}{\delta}
%\end{equation}

%If the second-order condition doesn't hold then for all $\bbx\in\ccalC$ and $\bbd_t^\top(\bbx-\bbx_t)
%\leq r$ we have 
%\begin{align}
% (\bbu_t-\bbx_t)^\top \bbH_t  (\bbu_t-\bbx_t)\geq -\rho\gamma'
%\end{align}
%Therefore, for all 

Now it remains to show that the outcome of Algorithm 2 is an $(\eps,\gamma)$-SOSP of Problem~\eqref{eq:stoc_main_problem} with high probability. Let's assume that $\bbx_t$ is the final output of Algorithm 2. Then, we know that $\bbx_t$ satisfies the conditions
%%%
\begin{equation}\label{proof_stochastic_part_2_100}
\bbd_t^{\top}(\bbx-\bbx_t)\geq-\eps' \quad \forall \ \bbx\in \ccalC,
\end{equation}
%%%
and 
%%%
\begin{equation}\label{proof_stochastic_part_2_200}
 (\bbx-\bbx_t)^\top \bbH_t  (\bbx-\bbx_t)\geq -\gamma' \quad \forall \ \bbx\in\ccalC,\ \bbd_t^\top(\bbx-\bbx_t)\leq r.
\end{equation}
 %%%
 First, we use the condition in \eqref{proof_stochastic_part_2_100} to show that $\bbx_t$ satisfies the first-order optimality condition with high probability. Note that for any $\bbx\in \ccalC$ it holds that
%%%
\begin{align}\label{proof_stochastic_part_2_300}
\nabla f(\bbx_t)^{\top} (\bbx-\bbx_t) 
&=\bbd_t^{\top} (\bbx-\bbx_t) + (\nabla f(\bbx_t)-\bbd_t)^{\top} (\bbx-\bbx_t) \nonumber\\
&\geq \bbd_t^{\top} (\bbx-\bbx_t) - D\|\nabla f(\bbx_t)-\bbd_t\|.
\end{align}
%%%
Now compute the minimum of both sides of \eqref{proof_stochastic_part_2_300} for all $\bbx\in \ccalC$ to obtain
%%%
\begin{align}\label{proof_stochastic_part_2_400}
\min_{\bbx\in\ccalC} \{\nabla f(\bbx_t)^{\top} (\bbx-\bbx_t) \}
&\geq \min_{\bbx\in\ccalC}\{\bbd_t^{\top} (\bbx-\bbx_t) - D\|\nabla f(\bbx_t)-\bbd_t\|\}\nonumber\\
&= \min_{\bbx\in\ccalC}\{\bbd_t^{\top} (\bbx-\bbx_t) \} - D\|\nabla f(\bbx_t)-\bbd_t\|\nonumber\\
&\geq -\eps' - D\|\nabla f(\bbx_t)-\bbd_t\|,
\end{align}
%%%
where the equality holds since $D\|\nabla f(\bbx_t)-\bbd_t\|$ does not depend on $\bbx$, and the last inequality is implied by \eqref{proof_stochastic_part_2_100}. Since $\E{\|\nabla f(\bbx_t)-\bbd_t\|^2}\leq \nu^2/b_g$ we obtain from Markov's inequality that 
\begin{equation}\label{proof_stochastic_part_2_500}
\mathbb{P}\left( \|\nabla f(\bbx_t)-\bbd_t\| \leq \eps''\right) \geq 1-\frac{\nu^2}{b_g\eps''^2}.
\end{equation}
%%%
Therefore, by combining the results in \eqref{proof_stochastic_part_2_400} and \eqref{proof_stochastic_part_2_500} we obtain that
%%%
\begin{equation}\label{proof_stochastic_part_2_600}
\mathbb{P}\left( \min_{\bbx\in\ccalC} \{\nabla f(\bbx_t)^{\top} (\bbx-\bbx_t) \} \geq -(\eps'+D\eps'')  \right) \geq 1-\frac{\nu^2}{b_g\eps''^2}.
\end{equation}
%%%
Now by setting $\eps''=\eps'/D$ it follows from \eqref{proof_stochastic_part_2_600} that with probability at least $1-\nu^2 D^2/b_g \eps'^2$ the final iterate $\bbx_t$ satisfies  
%%%
\begin{equation}\label{proof_stochastic_part_2_700}
\nabla f(\bbx_t)^{\top} (\bbx-\bbx_t) \geq -2\eps' \qquad \forall \ \bbx\in \ccalC.
\end{equation}
%%%
Replacing $\eps'$ by $\eps/2$ leads to 
%%%
\begin{equation}\label{proof_stochastic_part_2_702}
\nabla f(\bbx_t)^{\top} (\bbx-\bbx_t) \geq -\eps \qquad \forall \ \bbx\in \ccalC.
\end{equation}
%%%

It remains to show that with high probability the final iterate $\bbx_t$ satisfies the second-order optimality condition.

First, consider the sets $\ccalA_t=\{\bbx \mid \nabla f(\bbx_t)^{\top} (\bbx-\bbx_t) =0 \}$ and $\ccalB_t=\{ \bbx \mid \bbd_t^{\top}(\bbx-\bbx_t)\leq r \}$. We proceed to show that with high probability $\ccalA_t\subset\ccalB_t$. If $\bby$ satisfies the condition 
%%%
\begin{equation}\label{proof_stochastic_part_2_800}
 \nabla f(\bbx_t)^{\top} (\bby-\bbx_t) =0, 
\end{equation}
%%%
then it can be shown that 
%%%
\begin{align}\label{proof_stochastic_part_2_900}
\bbd_t^{\top} (\bby-\bbx_t)
&\leq  \nabla f(\bbx_t)^{\top} (\bby-\bbx_t)+ (\bbd_t-\nabla f(\bbx_t))^{\top} (\bby-\bbx_t) \nonumber\\
&\leq   D\|\bbd_t-\nabla f(\bbx_t)\|.
\end{align}
%%%
Since $\E{\|\nabla f(\bbx_t)-\bbd_t\|^2}\leq \nu^2/b_g$ we obtain from Markov's inequality that 
\begin{equation}\label{proof_stochastic_part_2_1000}
\mathbb{P}\left( \|\nabla f(\bbx_t)-\bbd_t\| \leq \frac{r}{D}\right) \geq 1-\frac{\nu^2D^2}{b_gr^2}.
\end{equation}
%%%%
Therefore, by combining the results in \eqref{proof_stochastic_part_2_900} and \eqref{proof_stochastic_part_2_1000} we obtain that
%%%
\begin{equation}\label{proof_stochastic_part_2_1100}
\mathbb{P}\left( \bbd_t^{\top} (\bby-\bbx_t) \leq r\right) \geq 1-\frac{\nu^2D^2}{b_gr^2}.
\end{equation}
%%%%
This argument shows that if $\bby\in \ccalA_t$, then it also belongs to the set $\ccalB_t$, i.e., $\bby\in \ccalB_t$, with high probability. This result shows if an inequality holds for all  $\bbx$ that satisfy $\bbd_t^{\top}(\bbx-\bbx_t)\leq r$, then with high probability that inequality also holds for all $\bbx$ that satisfy the condition $\nabla f(\bbx_t)^{\top} (\bbx-\bbx_t) =0$.

Now, note that if $\bbx_t$ is the output of Algorithm 2, then for any $\bbx\in \ccalC$ satisfying $\bbd_t^{\top} (\bbx-\bbx_t)\leq r$ it holds that
%%%%%%
\begin{align}\label{proof_stochastic_part_2_1200}
  (\bbx-\bbx_t)^\top \nabla^2 f(\bbx_t)  (\bbx-\bbx_t)  
 &= (\bbx-\bbx_t)^\top \bbH_t  (\bbx-\bbx_t)- (\bbx-\bbx_t)^\top( \bbH_t-\nabla^2 f(\bbx_t) ) (\bbx-\bbx_t)    \nonumber\\
 &\geq -\gamma' -D^2\|\bbH_t-\nabla^2 f(\bbx_t)\|.
\end{align}
%%%%%
Further, define the random variable $X_t=\|\bbH_t-\nabla^2 f(\bbx_t)\|$. As we know that $\E{X_t^2}\leq \xi^2/{b_H}$, it follows by Markov's inequality that $\mathbb{P}(X_t\leq a) \geq 1-\xi^2/({b_H}a^2)$. Therefore, we can write that
%%%
\begin{equation}\label{proof_stochastic_part_2_1300}
\mathbb{P}(\|\bbH_t-\nabla^2 f(\bbx_t)\|\leq \gamma'') \geq 1-\frac{\xi^2}{b_H\gamma''^2}.
\end{equation}
%%%
Hence, by using the results in \eqref{proof_stochastic_part_2_1200} and \eqref{proof_stochastic_part_2_1300}, we can show that with probability at least $1-\frac{\xi^2}{b_H\gamma''^2}$ for any $\bbx\in \ccalC$ satisfying $\bbd_t^{\top} (\bbx-\bbx_t)\leq r$ it holds
%%%%
\begin{align}
  (\bbx-\bbx_t)^\top \nabla^2 f(\bbx_t)  (\bbx-\bbx_t)  
 &\geq -\gamma' -D^2\gamma''.
\end{align}
%Therefore with high probability 
%\begin{align}
%\min_{\bbx\in\ccalC,\bbd_t^{\top} (\bbx-\bbx_t)\leq r}  (\bbx-\bbx_t)^\top \nabla^2 f(\bbx_t)  (\bbx-\bbx_t)  
% &\geq -\gamma' -D^2\gamma''/\rho
%\end{align}
%%%%
By setting $\gamma''= \gamma'/D^2$ it follows that $\bbx_t$ satisfies the condition 
%%%
\begin{align}
(\bbx-\bbx_t)^\top \nabla^2 f(\bbx_t)  (\bbx-\bbx_t)  \geq -2\gamma' \quad \forall\ \bbx\in\ccalC,\bbd_t^{\top} (\bbx-\bbx_t)\leq r,
\end{align}
%%%
with a probability larger than $1-\frac{\xi^2D^4}{b_H\gamma'^2}$. Further, with probability at least $1-\frac{\nu^2D^2}{b_gr^2}$ we know that $\ccalA_t\subset \ccalB_t$. These observations imply that if $\bbx_t$ is the output of Algorithm 2 it satisfies 
\begin{align}
(\bbx-\bbx_t)^\top \nabla^2 f(\bbx_t)  (\bbx-\bbx_t)  \geq -2\gamma' \quad \forall\ \bbx\in\ccalC,\nabla f(\bbx_t)^{\top} (\bbx-\bbx_t)=0,
\end{align}
%%%
with probability at least $1-\frac{\xi^2D^4}{b_H\gamma'^2}- \frac{\nu^2D^2}{b_gr^2}$, where we used the inequality 
\begin{align}
P(A \cap B) &= P(A)+P(B)-P(A\cup B)\nonumber\\
& \geq P(A)+P(B)-1.
\end{align}
%%%
By setting $\gamma'=\gamma/2$ we obtain that with probability at least $1-\frac{\xi^2D^4}{b_H\gamma'^2}- \frac{\nu^2D^2}{b_gr^2}$ the final iterate satisfies the condition 
\begin{align}
(\bbx-\bbx_t)^\top \nabla^2 f(\bbx_t)  (\bbx-\bbx_t)  \geq -\gamma \quad \forall\ \bbx\in\ccalC,\nabla f(\bbx_t)^{\top} (\bbx-\bbx_t)=0.
\end{align}
%%%

Therefore, with probability at least  $1-\frac{\nu^2 D^2}{b_g \eps'^2}-\frac{\xi^2D^4}{b_H\gamma'^2}- \frac{\nu^2D^2}{b_gr^2}$ the output of Algorithm 2 is an $(\eps,\gamma)$-SOSP of the stochastic optimization problem in \eqref{eq:stoc_main_problem}. This observation and the conditions on the batch sizes in \eqref{condition_on_batch_sizes} implies that  the output of Algorithm 2 is an $(\eps,\gamma)$-SOSP of the stochastic optimization problem in \eqref{eq:stoc_main_problem} with probability at least $1-\frac{1}{16}-\frac{1}{324}- \frac{\rho^2}{81}\geq 0.92$. (Note that $\rho\leq 1$).  Indeed, by increasing the size of batches $b_g$ and $b_H$ all the results hold with a higher probability.

\section*{Acknowledgment}

{This work was supported by DARPA Lagrange and ONR BRC Program. The authors would like to thank Yue Sun for pointing out a missing condition in the first draft of the paper.}

{{{
\bibliography{bibliography}
\bibliographystyle{abbrvnat}
}}}

\end{document}

%% file: abstract.tex
%!TEX root = root.tex
\begin{abstract}
In this paper, we study the problem of escaping from saddle points in smooth nonconvex optimization problems subject to a convex set $\ccalC$. We propose a generic framework that yields convergence to a second-order stationary point of the problem, if the convex set $\ccalC$ is simple for a quadratic objective function. Specifically, our results hold if one can find a $\rho$-approximate solution of a quadratic program subject to $\ccalC$ in polynomial time, where $\rho<1$ is a positive constant that depends on the structure of the set $\ccalC$. Under this condition, we show that the sequence of iterates generated by the proposed framework reaches an $(\eps,\gamma)$-second order stationary point (SOSP) in at most $\mathcal{O}(\max\{\eps^{-2},\rho^{-3}\gamma^{-3}\})$ iterations. {We further characterize the overall complexity of reaching an SOSP when the convex set $\ccalC$ can be written as a set of quadratic constraints and the objective function Hessian has a specific structure over the convex set $\ccalC$}. Finally, we extend our results to the stochastic setting and characterize the number of stochastic gradient and Hessian evaluations to reach an $(\eps,\gamma)$-SOSP.

%\red{consider the problem of finding stationary points of a generic smooth nonconvex stochastic optimization problem. To achieve this goal, we propose a variant of the celebrated stochastic gradient descent method which incorporates all the previous observed stochastic gradients to reduce the noise of stochastic approximation. Assuming the smoothness of only the expected function (the stochastic component functions might not be smooth), the proposed \red{Name} method succeeds in finding a stationary point $\hbx$ satisfying $\E{\|\nabla F(\hbx)\|^2}\leq \eps^2$ after $\mathcal{O}(1/\eps^{-3})$ iterations, i.e., after computing $\mathcal{O}(1/\eps^{-3})$ stochastic gradients. This result not only improves the $\mathcal{O}(1/\eps^{-4})$ bound for vanilla stochastic gradient descent, but also under less strict assumptions on the objective function outperforms the best known $\mathcal{O}(1/\eps^{-3.33})$ result for the finite-sum nonconvex optimization problem, which is a special case of the considered generic stochastic program. }

\end{abstract}